\documentclass{article}
\pdfoutput=1

\usepackage{times}
\usepackage[accepted]{icml2014} 

\usepackage{nicefrac}

\usepackage[colorlinks=true,linkcolor=blue,citecolor=blue]{hyperref}

\usepackage{mathtools} 
\usepackage{amsmath}
\usepackage{amssymb}
\usepackage{amsthm}
\usepackage{cancel}
\usepackage{graphicx}
\usepackage{cancel}
\usepackage{algorithm}
\usepackage{algorithmic}
\usepackage{xspace}
\usepackage{natbib}
\usepackage{subfigure}

\renewcommand\algorithmicthen{}
\renewcommand\algorithmicdo{}

\usepackage{color}
\usepackage[usenames,dvipsnames]{xcolor}

\usepackage{times}
\usepackage[backgroundcolor = White]{todonotes}

\usepackage{breakcites}
\usepackage[capitalize]{cleveref}

\newtheorem{theorem}{Theorem}
\newtheorem{lemma}{Lemma}

\newtheorem{corollary}{Corollary}

\newtheorem{assumption}{Assumption}[section]
\crefname{assumption}{Assumption}{assumptions}
\Crefname{assumption}{Assumption}{Assumptions}

\newcommand{\F}{\mathcal{F}}
\newcommand{\defined}{\doteq}
\newcommand{\indicate}[1]{{\mathbb{I}\{{#1}\}}}
\newcommand{\ind}[1]{{\mathbb{I}\{{#1}\}}}
\newcommand{\one}[1]{{\mathbb{I}\{{#1}\}}}

\newcommand{\var}{\mathbb{V}} 
\renewcommand{\P}{\mathbb{P}} 
\newcommand{\Prob}[1]{\mathbb{P}(#1)} 
\newcommand{\Probs}[1]{\mathbb{P}\left(#1\right)} 
\newcommand{\MC}{Monte Carlo\xspace}
\newcommand{\A}{\mathcal{A}}
\renewcommand{\epsilon}{\varepsilon}
\newcommand{\eps}{\varepsilon}

\newcommand{\R}{\mathbb{R}}
\newcommand{\real}{\mathbb{R}}
\newcommand{\N}{\mathbb{N}}
\newcommand{\E}{\mathbb{E}}
\newcommand{\EE}[1]{\E\left[#1\right]}
\newcommand{\V}{\var}
\newcommand{\Var}[1]{\V\left(#1\right)}
\renewcommand{\A}{\mathcal{A}}
\newcommand{\B}{\mathcal{B}}

\newcommand{\hmu}{\hat{\mu}}
\newcommand{\oX}{\overline{X}}

\newcommand{\ra}{\rightarrow}
\newcommand{\FF}{\mathcal{F}}
\newcommand{\bvar}{\bar{\var}}
\DeclareMathOperator{\KLname}{KL}
\newcommand{\KL}[1]{\KLname\left(#1\right)}
\DeclareMathOperator{\BanditAlg}{Bandit}
\DeclareMathOperator{\MCAlg}{MC}
\newcommand{\cset}[2]{\left\{#1\,:\,#2\right\}}

\icmltitlerunning{Adaptive \MC via Bandit Allocation}

\begin{document} 

\twocolumn[
\icmltitle{Adaptive \MC via Bandit Allocation}

\icmlauthor{\vspace{-0.15cm}}{\vspace{-0.15cm}}
\icmlauthor{James Neufeld}{jneufeld@ualberta.ca}
\icmlauthor{Andr\'as Gy\"orgy}{gyorgy@ualberta.ca}
\icmlauthor{Dale Schuurmans}{daes@ualberta.ca}
\icmlauthor{Csaba Szepesv\'ari}{csaba.szepesvari@ualberta.ca}
\icmladdress{Department of Computing Science, University of Alberta, Edmonton, AB, Canada T6G 2E8}
\icmlkeywords{Online Learning, multi-armed bandit, Monte Carlo, Regret Analysis}
]

\begin{abstract}

We consider the problem of sequentially choosing between a set of
unbiased \MC estimators to minimize the mean-squared-error (MSE) of a
final combined estimate. 
By reducing this task to a \emph{stochastic multi-armed bandit} problem,
we show that well developed allocation strategies can be used to achieve 
an MSE that approaches that of the best estimator chosen in retrospect.
We then extend these developments to a scenario where alternative estimators 
have different, possibly stochastic costs.
The outcome is a new set of adaptive \MC strategies that provide stronger
guarantees than previous approaches while offering practical advantages.

\end{abstract}

\section{Introduction}

\MC methods are a pervasive approach to approximating complex integrals,
which are widely deployed in all areas of science.
Their widespread adoption has led to the development dozens
of specialized \MC methods for any given task, 
each having their own tunable parameters.  
Consequently, it is usually difficult for a practitioner to know which 
approach and corresponding parameter setting
might be most effective for a given problem.

In this paper we develop algorithms for sequentially allocating calls
between a set of unbiased estimators to minimize the expected squared error 
(MSE) of a combined estimate.  
In particular, we formalize a new class of adaptive estimation problem:
\emph{learning to combine \MC estimators}.  
In this scenario, one is given a set of \MC estimators that can each 
approximate the expectation of some function of interest.  
We assume initially that each estimator is unbiased 
but has \emph{unknown} variance. 
In practice, such estimators could include any unbiased
method and/or variance reduction technique, such as unique instantiations
of importance, stratified, or rejection sampling; antithetic variates;
or control variates \citep{Robert2005}.
The problem is to design a sequential allocation procedure that can
interleave calls to the estimators and combine their outputs
to produce a combined estimate whose MSE decreases as quickly as possible.
To analyze the performance of such a
meta-strategy we formalize the notion of \emph{MSE-regret}: the
time-normalized excess MSE of the combined estimate compared to the best 
estimator selected in hindsight, 
i.e., with knowledge of the distribution of estimates produced by each
base estimator.

Our first main contribution is to show that this meta-task can be
reduced to a \emph{stochastic multi-armed bandit problem}, where
bandit arms are identified with base estimators and the payoff of an
arm is given by the negative square of its sampled estimate.  In
particular, we show that the MSE-regret of \emph{any} meta-strategy is
equal to its bandit-regret when the procedure is used to play in the
corresponding bandit problem.
As a consequence, we conclude that existing bandit algorithms, 
as well as their bounds on bandit-regret, can be
immediately applied to achieve new results for adaptive \MC
estimation.  Although the underlying reduction is quite simple, the
resulting adaptive allocation strategies provide novel 
alternatives to traditional \emph{adaptive} \MC strategies,
while providing strong finite-sample performance guarantees.

Second, we consider a more general case where the alternative
estimators require different (possibly random) costs to
produce their sampled estimates.  Here we develop a suitably designed
bandit formulation that yields bounds on the MSE-regret for cost-aware
estimation.  
We develop new algorithms for this generalized form of adaptive \MC,
provide explicit bounds on their MSE-regret,
and compare their performance to a state-of-the-art adaptive \MC method.
By instantiating a set of viable base estimators and selecting from
them algorithmically, rather than tuning parameters manually, we
discover that both computation and experimentation time can be reduced.

This work is closely related, and complementary to
work on adaptive stratified sampling  \citep{CaMu11},
where a strategy is designed to allocate samples between fixed strata 
to achieve MSE-regret bounds relative to the best allocation proportion
chosen in hindsight.  
Such work has since been extended to optimizing the number \citep{CaMu12:ALT} 
and structure \citep{CaMu12:NIPS} of strata for differentiable functions.  
The method proposed in this
paper, however, can be applied more broadly to any set of base estimation 
strategies and potentially even in combination with these approaches.


\section{Background on Bandit Problems}
\label{sec:background}


The multi-armed bandit (MAB) problem is a sequential allocation task
where an agent must choose an action at each step to maximize
its long term payoff, when only the payoff of the selected action can
be observed \citep{cesa06, bubeck2012regret}.
%
%
In the \emph{stochastic} MAB problem \citep{robbins1952} 
the payoff for each action $k \in \{1,2,\ldots,K\}$ 
is assumed to be generated independently and identically (i.i.d.)  
from a fixed but unknown distribution $\nu_k$.
The performance of an allocation policy can then by analyzed by
defining the \emph{cumulative regret} for any sequence of $n$ actions,
given by
%
%
\vspace{-0.15cm}
\begin{equation}
R_n \defined \max_{1\le k\le K} 
\E\Big[\sum_{t=1}^n X_{k,t} - \sum_{t=1}^n X_{I_t,T_{I_t}(t)}
\Big],
\label{eqn:regret}
\vspace{-0.15cm}
\end{equation}
where $X_{k,t} \in \R$ is the random variable giving the $t$th payoff
of action $k$, $I_t \in \{1,\ldots,K\}$ denotes the action taken by
the policy at time-step $t$, 
and $T_k(t) \defined \sum_{s=1}^t \indicate{I_s = k}$ 
denotes the number of times action $k$ is chosen by the policy up to time $t$.  
Here, $\indicate{p}$ is the indicator
function, set to 1 if the predicate $p$ is true, 0 otherwise.  
The objective of the agent is to maximize the total payoff, or
equivalently to minimize the cumulative regret.  
By rearranging \eqref{eqn:regret} and conditioning, the regret can be rewritten

\vspace{-0.8cm}

\begin{align}
\label{eq:regretdecomposition}
R_n 
= \sum_{k=1}^K \E[T_k(n)](\mu^* -\mu_k),
\end{align}

\vspace*{-0.5cm}

\noindent
where $\mu_k\defined\E[X_{k,t}]$ and
$\mu^* \defined \max_{j=1,\ldots,K} \mu_j$.

The analysis of the stochastic MAB problem was pioneered by
\citet{Lai85} who showed that, when the payoff distributions are
defined by a single parameter, the asymptotic regret of any
sub-polynomially consistent policy 
(i.e., a policy that selects non-optimal actions only sub-polynomially 
many times in the time horizon) 
is lower bounded by $\Omega(\log n)$.  
In particular, for Bernoulli payoffs
\vspace{-0.1cm}
\begin{equation}
  \liminf_{n \to \infty} \frac{R_n}{\log n} \geq \sum_{k :
    \Delta_k > 0}\frac{\Delta_k}{\text{KL}(\mu_k,\mu^*)},
\label{eq:lb}
\vspace{-0.25cm}
\end{equation}
where $\Delta_k \defined \mu^* - \mu_{k}$ and $\KL{p,q} \defined p
\log(p/q) + (1-p) \log( \frac{1-p}{1-q} )$ for $p,q\in [0,1]$.
\citet{Lai85} also presented an algorithm based on upper confidence
bounds (UCB), which achieves a regret asymptotically matching the
lower bound (for certain parametric distributions).

Later, \citet{aueretal02} proposed UCB1 (Algorithm~\ref{alg:ucb1}), 
which broadens the practical use of UCB by dropping the requirement that payoff
distributions fit a particular parametric form.
Instead, one need
only make the much weaker assumption that the rewards are bounded; in
particular, we let $X_{k,t} \in [0,1]$.  
\citet{aueretal02} proved that, for any finite number of actions $n$,
UCB1's regret is bounded by
%
\vspace{-0.1cm}
\begin{equation}
R_n \leq \sum_{k: \Delta_k>0} \frac{8\log n}{\Delta_k}
+ \Big(1 + \frac{\pi^2}{3}\Big) \Delta_k.
\label{eq:ucb1 ub}
\end{equation}
\vspace{-0.6cm}

\begin{algorithm}[t]
\begin{algorithmic}[1]
\FOR{$k \in\{1, \ldots, K\}$}
\STATE 
Play $k$, observe $X_{k,1}$, set $\bar{\mu}_{k,1} := X_{k,1}$; $T_k(1) := 1$.
\ENDFOR
\FOR{$t \in\{K+1, K+2, \ldots\}$}
\STATE 
Play action $k$ that maximizes 
$\bar{\mu}_{j,t-1} + \sqrt{\frac{2\log t}{T_j(t-1)}}$; 
set $T_k(t)=T_k(t-1)+1$ and $T_j(t)=T_j(t-1)$ for $j\neq k$, 
observe payoff $X_{k,T_k(t)}$, and compute
$\bar{\mu}_{k,t} = (1-1/T_k(t)) \bar{\mu}_{k,t-1} + X_{k,T_k(t)}/T_k(t)$.
\ENDFOR
\end{algorithmic}
\caption{UCB1 \citep{aueretal02}}
\label{alg:ucb1}
\end{algorithm}

Various improvements of the UCB1 algorithm have since been proposed.
One approach of particular interest is the UCB-V algorithm
\citep{audibert2009exploration}, which takes the empirical variances
into account when constructing confidence bounds.  Specifically, UCB-V
uses the bound
$$
\bar{\mu}_{k,t} + \sqrt{\frac{2\var_{k,T_k(t-1)}\mathcal{E}_{T_k(t-1),t}}{T_k(t-1)}} +
c\frac{3 \mathcal{E}_{T_k(t-1),t}}{T_{k}(t-1)},
$$
where $\var_{k,s}$ denotes the empirical variance of arm $k$'s payoffs
after $s$ pulls, $c>0$, and $\mathcal{E}_{s,t}$ is an
\emph{exploration function} required to be a non-decreasing function
of $s$ or $t$ (typically $\mathcal{E}_{s,t} = \zeta\log(t)$ for a
fixed constant $\zeta > 0$).  The UCB-V procedure can then be
constructed by substituting the above confidence bound into
Algorithm~\ref{alg:ucb1}, which yields a regret bound that scales with
the true variance of each arm
\begin{equation}
R_n \leq c_\zeta \sum_{k: \Delta_k>0} 
\Big(\frac{\var(X_{k,1})}{\Delta_k} + 2\Big) \log n
.
\vspace{-0.2cm}
\label{eq:ucbv ub}
\end{equation}
Here $c_\zeta$ is a constant relating to $\zeta$ and $c$.
In the worst case, when $\var(X_{k,1}) = 1/4$ and $\Delta_k=1/2$, 
this bound is slightly worse than UCB1's bound;
however, it is usually better in practice, particularly
if some $k$ has small $\Delta_k$ and $\var(X_{k,1})$.

A more recent algorithm is KL-UCB \citep{Cappe2013},
where the confidence bound for arm $k$ is based on solving
\vspace{-0.1cm}
\[
\sup\Big\{ \mu: \KL{\bar{\mu}_{k,t},\mu} \le \frac{ f(t)}{ T_k(t)} \Big\}\,,
\vspace{-0.2cm}
\]
for a chosen increasing function $f(\cdot)$,
which can be solved efficiently
since $\KL{p,\cdot}$ is smooth and increasing on $[p,1]$. 
By choosing $f(t)\!=\!\log(t)+3\log\log(t)$ for $t\!\ge\! 3$ 
(and $f(1)\!\!\!=\!\!f(2)\!\!=\!\!f(3)$),
KL-UCB achieves a regret bound
\vspace{-0.1cm}
\begin{equation}
R_n \le \sum_{k: \Delta_k>0} 
\Big( \frac{\Delta_k}{\KL{\mu_k,\mu^*}} \Big) \log n + O(\sqrt{\log(n)})
\vspace{-0.1cm}
\label{eq:ucbkl ub}
\end{equation}
for $n\ge 3$,
with explicit constants for the ``higher order'' terms 
\citep[Corollary 1]{Cappe2013}.
Apart from the higher order terms, this bound matches the lower bound 
\eqref{eq:lb}.
In general, 
KL-UCB is expected to be better than UCB-V except for large sample sizes
and small variances. 
%
%
%
Note that,
given any set of UCB algorithms,
one can apply the tightest upper confidence from the set, 
via the union bound, 
at the price of a small additional constant in the regret.

Another approach that has received significant recent interest
is Thompson sampling~(TS) \citep{thompson1933likelihood}:
a Bayesian method where actions are chosen randomly in proportion
to the posterior probability that their mean payoff is optimal.
TS is known to outperform UCB-variants when payoffs are Bernoulli distributed
\citep{chapelle2011empirical,may2011simulation}.  
Indeed, the finite time regret of TS under Bernoulli payoff distributions
closely matches the lower bound 
\eqref{eq:lb}
\citep{kaufmann2012thompson}:
\vspace*{-0.1cm}
$$
R_n \leq (1+\epsilon)\!\!\sum_{k: \Delta_k>0} \!\!
\frac{\Delta_k(\log(n) +
  \log\log(n))}{\text{KL}(\mu_k,\mu^*)} + C(\epsilon,\mu_{1:K})
,
\vspace*{-0.1cm}
$$
for every $\epsilon > 0$, where $C$ is a problem-dependant constant.
However, since it is not possible to have Bernoulli distributed
payoffs with the same mean and different variances, 
this analysis is not directly applicable to our setting. 
Instead, we consider a more general version of Thompson sampling 
\citep{AgGo12} that converts real-valued to Bernoulli-distributed payoffs 
through a resampling step (Algorithm~\ref{alg:ts}),
which has been shown to obtain
\vspace*{-0.2cm}
\begin{equation}
R_n \leq \Big(\sum_{k: \Delta_k>0} \frac{1}{\Delta_k^2}\Big)^2\log(n)
.
\label{eq:ts ub}
\end{equation}

\vspace*{-0.4cm}

\begin{algorithm}[t]
   \caption{Thompson Sampling \small\citep{AgGo12}}
   \label{alg:ts}
\begin{algorithmic}
  \STATE \textbf{Require: } Prior parameters $\alpha$ and $\beta$
  \STATE \textbf{Initialize: } $S_{1:K}(0) := 0$, $F_{1:K}(0) := 0$
  \FOR{$t \in \{1,2,\ldots\}$ }
  \STATE Sample $\theta_{t,k}\!\sim\!\mathcal{B}(S_k(t\!-\!1)\!+\!\alpha,F_k(t\!-\!1)\!+\!\beta)$, $k=1...K$
  \STATE Play $k = \arg\max_{j=1,\ldots,K} \theta_{t,j}$; observe $X_t\in[0,1]$
  \STATE Sample $\hat{X}_t \sim \text{Bernoulli}(X_t)$
  \STATE \textbf{if} $\hat{X}_t = 1$ \textbf{then} set $S_k(t) = S_{k}(t-1)+1$
  \textbf{else} set $F_{k}(t) = F_{k}(t)+1$ 
  \ENDFOR
\end{algorithmic}
\end{algorithm}

\section{Combining \MC Estimators}
\label{sec:mc}

We now formalize the main problem we consider in this paper.
Assume we are given a finite number of \MC estimators, $k=1,\ldots,K$,
where base estimator $k$ produces a sequence of real-valued 
random variables $(X_{k,t})_{(t=1,2,\ldots)}$
whose mean converges to the unknown target quantity, $\mu\in \real$. 
Observations from the different estimators are assumed to be independent 
from each other.
We assume, initially, that drawing a sample from each estimator takes 
constant time, hence the estimators differ only in terms of how fast 
their respective sample means $\oX_{k,n} = \frac1n \sum_{t=1}^n X_{k,t}$ 
converge to $\mu$.
The goal is to design a \emph{sequential estimation procedure} that works 
in discrete time steps:
For each round $t=1,2,\ldots$,
based on the previous observations,
the procedure selects one estimator $I_t \in \{1, \ldots, K\}$,
whose observation is used by an outer procedure to update the
estimate $\hat{\mu}_{t}\in \real$ based on the values observed so far.

As is common in the \MC literature, we evaluate accuracy by the
mean-squared error (MSE).  That is, we define the loss of the
sequential method $\A$ at the end of round $n$ by $L_n(\A) = \EE{
  (\hmu_n - \mu)^2 }$.  A reasonable goal is to then compare the loss,
$L_{k,n} = \EE{ (\oX_{k,n}-\mu)^2 }$, of each base estimator 
to the loss of $\A$.  
In particular, we propose to evaluate the performance of $\A$
by the (normalized) regret \vspace*{-0.1cm}
\[
R_n({\A}) = n^2 \Big( L_n({\A}) - \min_{1\le k\le K} L_{k,n} \Big)\,,
\vspace*{-0.1cm}
\]
which measures the excess loss of $\A$ due to its initial ignorance of
estimator quality.  Implicit in this definition is the assumption that
$\A$'s time to select the next estimator is negligible compared to the
time to draw an observation.
Note also that the excess loss is multiplied by $n^2$, which ensures
that, in standard settings, when $L_{k,n} \propto 1/n$, a sublinear
regret (i.e., $|R_n(\A|)/n \ra 0$ as $n\ra\infty$) implies that the loss of
$\A$ asymptotically matches that of the best estimator.

%


In the next two sections we will adopt a simple strategy for combining
the values returned from the base estimators: $\A$ simply returns
their (unweighted) average as the estimate $\hmu_n$ of $\mu$. A more
sophisticated approach may be to weight each of these samples
inversely proportional to their respective (sample) variances.
However, if the adaptive procedure can quickly identify and ignore
highly suboptimal arms the savings from the weighted estimator will
diminish rapidly. Interestingly, this argument does not immediately
translate to the nonuniform cost case considered in
\cref{sec:nonuniformcost} as will be shown empirically in
\cref{sec:exp}.
%
%

\section{Combining Unbiased I.I.D. Estimators}
\label{sec:uniformcost}

Our main assumption in this section will be the following:
\begin{assumption}\label{ass:uniiid}
Each estimator produces a sequence of i.i.d.\ random observations
with common mean $\mu$ and finite variance;
values from different estimators are independent.
\end{assumption}
Let $\psi_k$ denote the distribution of samples from estimator $k$.
Note that $\Psi = (\psi_k)_{1\le k \le K}$ completely determines the
sequential estimation problem.
%
%
%
Since the samples coming from estimator $k$ are i.i.d., we have
$\Var{X_{k,1}} = \Var{X_{k,t}}$.  Let $V_k = \Var{X_{k,1}}$ and $V^* =
\min_{1\le k \le K} V_k$.  Furthermore, let $L_{k,t} =
\Var{X_{k,1}}/t$, hence $\min_{1\le k \le K} L_{k,t} = V^*/t$.
We then have the first main result of this section.

\begin{theorem}[Regret Identity]\label{thm:regretmatch}
Consider $K$ estimators for which  \Cref{ass:uniiid} holds,
and
let $\A$ be an arbitrary allocation procedure.
Then, for any $n\ge 1$,
the MSE-regret of the estimation procedure $\A^{\rm avg}$, estimating $\mu$
using the sample-mean of the observations obtained by $\A$, satisfies
\vspace*{-0.1cm}
\begin{equation}
R_n(\A^{\rm avg})  
= \sum_{k=1}^K \EE{T_k(n)} (V_k- V^* )\,.
\label{eqn:regret2}
\vspace*{-0.4cm}
\end{equation}
\end{theorem}

The proof follows from a simple calculation given in \cref{sec:apx-uniform}.
Essentially, one can rewrite the loss as
 $L_n(\A) = 
 \frac{1}{n^2}\E\left[\sum_{k=1}^K S_{k,n}^2 +
    2\sum_{k \ne j} S_{k,n}S_{j,n}\right]$,
where $S_{k,n}$ is the centered sum of observations for arm $k$.
The cross-terms can be shown to cancel by independence,
and Wald's second identity with some algebra gives the result.

The tight connection between sequential estimation
and bandit problems revealed by~\eqref{eqn:regret2}
 allows one to reduce sequential estimation to the design 
of bandit strategies and vice versa. 
Furthermore, regret bounds transfer both ways.

\begin{theorem}[Reduction]
\label{thm:reduction}
Let \Cref{ass:uniiid} hold for $\Psi$.
Define a corresponding bandit problem $(\nu_k)$
by assigning $\nu_k$ as the distribution of $-X_{k,1}^2$.
%
%
Given an arbitrary allocation strategy $\A$,
let $\BanditAlg(\A)$ be the bandit strategy 
that consults $\A$ to select the next arm
after obtaining reward $Y_t$ (assumed nonpositive),
based on feeding observations $(-Y_t)^{1/2}$ to $\A$ 
and copying $\A$'s choices.
Then, the bandit-regret of $\BanditAlg(\A)$ in bandit problem $(\nu_k)$
is the same as the MSE-regret of $\A$ in estimation problem $\Psi$.
%
%
Conversely, 
given an arbitrary bandit strategy $\B$,
let $\MCAlg(\B)$ be the allocation strategy 
that consults $\B$ to select the next estimator
after observing $-Y_t^2$,
based on feeding rewards $Y_t$ to $\B$ 
and copying $\B$'s choices.
Then the MSE-regret of $\MCAlg(\B)$ in estimation problem $\Psi$ 
is the same as the bandit-regret of $\B$ in bandit problem $(\nu_k)$
(where $\MCAlg(\B)$ uses the average of observations as its estimate).
\end{theorem}
\vspace*{-0.2cm}
\begin{proof}[Proof of \cref{thm:reduction}]
The result follows from \cref{thm:regretmatch} since
$V_k = \E[X_{k,1}^2] - \mu^2$ and $V^* = \E[X_{k^*,1}^2]- \mu^2$
where $k^*$ is the lowest variance estimator,
hence $V_k - V^* = \E[X_{k,1}^2] - \min_{1\le k'\le K} \E[X_{k',1}^2]$.
Furthermore, the bandit problem $(\nu_k)$ ensures the regret of a procedure 
that chooses arm $k$ $T_k(n)$ times is $\sum_{k=1}^K \E[T_k(n)] \Delta_k$,  
where
$\Delta_k = \max_{1\le k' \le K}\E[-X_{k',1}^2] - \E[-X_{k,1}^2] = V_k - V^*$.
\end{proof}
\vspace*{-0.2cm}

From this theorem one can also derive a lower bound.
First, let $\V(\psi)$ denote the variance of $X\!\sim\!\psi$ and
let $\V^*(\Psi) = \min_{1\le k \le K} \V(\psi_k)$.
For a family $\FF$ of  distributions over the reals,
let $D_{\inf}(\psi,v,\FF) = \inf_{\psi'\in \FF: \V(\psi')<v} D(\psi,\psi')$, 
where 
$D(\psi,\phi) =\int \log \frac{d\psi}{d\phi}(x) \,d\psi(x)$,
if the Radon-Nikodym derivative $d\psi/d\phi$ exists, 
and $\infty$ otherwise.
Note that $D_{\inf}(\psi,v,\FF)$ measures how distinguishable $\psi$ is from 
distributions in $\FF$ having smaller variance
than $v$. Further, we let $R_n(\A,\Psi)$ denote the regret of $\A$ on the 
estimation problem specified using the distributions $\Psi$.
\begin{theorem}[MSE-Regret Lower Bound]
Let $\FF$ be the set of distributions supported on $[0,1]$ and assume that
$\A$ allocates a subpolynomial fraction to suboptimal estimators
for any $\Psi\!\in\!\FF^K$:
i.e.,
$\E_{\Psi}[T_k(n)] \!=\! O(n^a)$ for all $a>0$ and $k$ such that 
$\V(\psi_k)\!>\!\V^*(\Psi)$.
Then, for any $\Psi \!\in\! \FF$ where not all variances are equal and 
$0\!<\!D_{\inf}(\psi_k,\V^*(\Psi),\FF)\!<\!\infty$ holds whenever 
$\V(\psi_k)\!>\!\V^*(\Psi)$,
we have
\vspace*{-0.1cm}
\[
\lim\inf_{n \ra \infty} \frac{R_n(\A,\Psi)}{\log n} \ge
\sum_{k: \V(\psi_k)>\V^*(\Psi)} 
 \frac{\V(\psi_k)-\V^*(\Psi)}{D_{\inf}(\psi_k,\V^*(\Psi),\FF)}.
\vspace*{-0.2cm}
\]
\end{theorem}
\begin{proof}
The result follows from \cref{thm:reduction} and \citep[Proposition 1]{BuKa96}.
\end{proof}
\vspace*{-0.2cm}
Using \cref{thm:reduction} we can also establish bounds on the MSE-regret for
the algorithms mentioned in \cref{sec:background}.
\begin{theorem}[MSE-Regret Upper Bounds]
Let \Cref{ass:uniiid} hold for $\Psi\!=\!(\psi_k)$ 
where (for simplicity) we assume each $\psi_k$ is supported on $[0,1]$.
Then, after $n$ rounds, 
$\MCAlg(\B)$ 
achieves the MSE-Regret bound of:
\eqref{eq:ucb1 ub} when using $\B\!=\!\text{UCB1}$;
\eqref{eq:ucbv ub} when using $\B\!=\!\text{UCB-V}$ with $c_\zeta=10$;
\eqref{eq:ucbkl ub} when using $\B\!=\!\text{UCB-KL}$;
and
\eqref{eq:ts ub} when using $\B\!=\!\text{TS}$.%
\footnote{
Note that to apply the regret bounds from Section~\ref{sec:background}, 
one has to feed the bandit algorithms with 
$1-Y_t^2$ instead of $-Y_t^2$ in Theorem~2. 
This modification has no effect on the regret.
}
\end{theorem}
\vspace*{-0.2cm}
Additionally, due to \cref{thm:reduction}, one can also obtain bounds
on the minimax MSE-regret by exploiting the lower bound for bandits in
\citep{ACFS:2002}.  In particular, the UCB-based bandit algorithms
above can all be shown to achieve the minimax rate $O(\sqrt{Kn})$ up
to logarithmic factors, immediately implying that the minimax
MSE-regret of $\MCAlg(\B)$ for $\B \in \{ \text{UCB1}, \text{UCB-V},
\text{KL-UCB}\}$ is of order $L_n(\MCAlg(\B))-L_n^* =
\tilde{O}(K^{1/2} n^{-(1+\frac1{2})})$.%
\footnote{
$\tilde{O}$ denotes the order up to logarithmic factors.
To remove such factors one can exploit MOSS \citep{AuBu10}.
}

\cref{sec:apx-range} provides a discussion of
how alternative ranges on the observations can be handled,
and how the above methods can still be 
applied when the payoff distribution is unbounded but satisfies
moment conditions.

\section{Non-uniform Estimator Costs}
\label{sec:nonuniformcost}

Next, we consider the case when the base estimators can take
\emph{different} amounts of time to generate observations.  A
consequence of non-uniform estimator times, which we refer to as
\emph{non-uniform costs}, is that the definitions of the loss and
regret must be modified accordingly.  Intuitively, if an estimator
takes more time to produce an observation, it is less useful than
another estimator that produces observations with (say) identical
variance but in less time.

To develop an appropriate notion of regret for this case, we introduce
some additional notation.  Let $D_{k,m}$ denote the time needed by
estimator $k$ to produce its $m$th observation, $X_{k,m}$.  As before,
we let $I_m\!\in\! \{1,\ldots,K\}$ denote the index of the estimator
that $\A$ chooses in round $m$.  Let $J_m$ denote the time when $\A$
observes the $m$th sample, $Y_m \!=\! X_{I_m,T_{I_m}(m)}$; thus, $J_1
\!=\! D_{I_1,1}$, $J_2 \!=\! J_1\!+\!D_{I_2,T_{I_2}(2)}$, and $J_{m+1}
\!=\! J_{m}\!+\!D_{I_m,T_{I_m}(m)} \!=\! \sum_{s=1}^m
D_{I_s,T_{I_s}(s)}$.  For convenience, define $J_0 \!=\! 0$.  Note
that round $m$ starts at time $J_{m-1}$ with $\A$ choosing an
estimator, and finishes at time $J_m$ when the observation is received
and $\A$ (instantaneously) updates it estimate.  Thus, at time $J_m$ a
new estimate $\hmu_m$ becomes available: the estimate is ``renewed''.
Let $\hat{\mu}(t)$ denote the estimate available at time $t\ge 0$.
Assuming $\A$ produces a default estimate $\hmu_0$ before the first
observation, we have $\hmu(t)\!=\!\hmu_0$ on $[0,J_1)$, $\hat{\mu}(t)
\!=\! \hmu_1$ on $[J_1,J_2)$, etc.  If $N(t)$ denotes the round index
at time $t$ (i.e., $N(t)\!=\!1$ on $[0,J_1)$, $N(t)\!=\!2$ on
$[J_1,J_2)$, etc.)  then $\hmu(t) \!=\! \hmu_{N(t)-1}$.  The MSE of
$\A$ at time $t\ge 0$ is \vspace*{-0.1cm}
\[
L(\A,t) = \EE{ (\hmu(t)-\mu)^2 }\,.
\vspace*{-0.1cm}
\]
By comparison, the estimate for a single estimator $k$ at time $t$ is 
$\hmu_k(t) = \one{N_k(t)>1} \frac{\sum_{m=1}^{N_k(t)-1} X_{k,m} }{N_k(t)-1}$,
where $N_k(t)\!=\!1$ on $[0,D_{k,1})$, 
$N_k(t)\!=\!2$ on $[D_{k,1},D_{k,1}+D_{k,2})$, etc.
We set $\hmu_k(t)\!=\!0$ on $[0,D_{k,1})$ 
to let $\hmu_k(t)$ be well-defined on $[0,D_{k,1})$.
Thus, at time $t\ge0$ the MSE of estimator $k$  is 
\vspace*{-0.1cm}
\[
L_k(t) = \EE{ (\hmu_k(t) - \mu)^2 }\,.
\vspace*{-0.1cm}
\]
Given these definitions, it is natural to define the regret as
\vspace*{-0.1cm}
\[
R(\A,t) = t^2 \Big( L(\A,t) - \min_{1\le k \le K} L_k(t) \Big) \,.
\vspace*{-0.1cm}
\]
As before, the $t^2$ scaling is chosen so that, 
under the condition that $L_k(t) \propto 1/t$, 
a sublinear regret implies that $\A$ is ``learning''. 
Note that this definition generalizes the previous one:
if $D_{k,m}=1\,\forall k$, $m$, then $R_n(\A)=R(\A,n)$.

In this section we make the following assumption.
\begin{assumption}
\label{ass:cost}
For each $k$,
$(X_{k,m},D_{k,m})_{(m=1,2, \ldots)}$ 
is an i.i.d.\ sequence
such that
$\EE{X_{k,1}}\!=\!\mu$, 
$V_k \!\defined \! \Var{X_{k,1}}\!<\!\infty$,
$\Prob{D_{k,m}\!>\!0}\!=\!1$
and
$\delta_k\!\defined\!\E[D_{k,1}]\!=\!\EE{D_{k,m}}\!<\!\infty$.
Furthermore, we assume that the sequences 
for different $k$ are independent of each other.
\end{assumption}

Note that \cref{ass:cost} allows $D_{k,m}$ to be a deterministic value;
a case that holds when the estimators use deterministic algorithms to produce 
observations. 
Another situation arises when $D_{k,m}$ is stochastic 
(i.e., the estimator uses a randomized algorithm) 
and $(X_{k,m},D_{k,m})$ are correlated.
In which case $\hmu_k(t)$ may be a biased estimate of $\mu$.
However, if $(X_{k,m})_m$ and $(D_{k,m})_m$ are independent 
and $\Prob{N_k(t)\!>\! 1}\!=\!1$ 
then $\hmu_k(t)$ is unbiased.
Indeed, in such a case, $(N_k(t))_t$ is independent of the partial 
sums $(\sum_{m=1}^n X_{k,m})_n$, 
hence 
$\E[\hmu_k(t)]
=
\E\Big[\frac{\sum_{m=1}^{N_k(t)-1} X_{k,m} }{N_k(t)-1}\Big]
=
\sum_{n=2}^\infty \Prob{ N_k(t)\!=\!n } 
\E\Big[ \frac{\sum_{m=1}^{n-1} X_{k,m} }{n-1} \Big| N_k(t) \!=\! n \Big]
=
\Prob{N_k(t)\!>\!1}\, \EE{X}
$,
because $\hmu_k(t)\!=\!0$ when $N_k(t)\!\le\! 1$.

%
%

Using \Cref{ass:cost}, a standard argument of \emph{renewal reward processes}
gives $L_k(t)\! \sim \!V_k/(t/\delta_k) \!=\! V_k \delta_k/t$, 
where $f(t)\!\!\sim\!\!g(t)$ means $\lim_{t\ra\infty} f(t)/g(t) \!=\! 1$.  
Intuitively, estimator $k$ will produce approximately $t/\delta_k$ 
independent observations during $[0,t)$;
hence, the variance of their average is approximately $V_k/(t/\delta_k)$ 
(made more precise in the proof of \cref{thm:nonunifreduction}).
This implies
$\min_{1\le k \le K} L_k(t) \!\sim\! 
\min_{1\le k \le K} \frac{\delta_k V_k }{t}$.
Thus, any allocation strategy $\A$ competing
with the best estimator must draw most of its observations from 
$k$ satisfying
$\delta_k V_k \!=\! \delta^* V^*
\!\defined\!\min_{1\le k \le K} \frac{\delta_k V_k }{t}$. 
For simplicity, we assume 
$k^*\!=\!{\arg\min}_{1\le k \le K} \delta_k V_k$ is unique, 
with $\delta^*\!=\!\delta_{k^*}$ and $V^*=V_{k^*}$.


As before, we will consider adaptive strategies that estimate 
$\mu$ using the mean of the observations:
$\hmu_m \!=\! \frac{S_m}{m}$ such that $S_m\!\defined\!\sum_{s=1}^m Y_s$.
Hence, the estimate at time $t\ge J_1$ is
\vspace*{-0.1cm}
\begin{equation}
\label{eq:hmu}
\hmu(t) = \frac{S(t)}{N(t)-1}, \text{ where } S(t) = S_{N(t)-1}\,.
\vspace*{-0.1cm}
\end{equation}
Our aim is to bound regret of the overall algorithm by bounding the 
number of times the allocation strategy chooses suboptimal estimators. 
We will do so by generalizing \eqref{eq:regretdecomposition} 
to the nonuniform-cost case, but unlike the equality obtained 
in \cref{thm:regretmatch}, here we provide an upper bound. 


\begin{theorem}
\label{thm:nonunifreduction}
Let~\Cref{ass:cost} hold and assume that $(X_{k,m})$ are bounded 
%
%
and $k^*$ is unique.
Let the estimate of $\A$ at time $t$ be defined by the sample mean $\hmu(t)$.
Let $t\!>\!0$ be such that $\EE{N(t)-1}\!>\!0$ and $\EE{N_{k^*}(t)}\!>\!0$, 
and assume that for any $k\!\neq\! k^*$,
$\EE{T_k(N(t))}\!\le\! f(t)$ for some $f\!:\!(0,\infty)\! \to\! [1,\infty)$
such that $f(t)\!\le\! c_f t$ for some $c_f\!>\!0$ and any $t\!>\!0$.
Assume furthermore that $\Prob{D_{k,1}\!>\!t}\! \le\! C_D t^{-2}$ 
and $\EE{N_{k^*}(t)^2} \!\le\! C_N t^2$ for all $t \!>\! 0$.
Then, for any $c \!<\! \sqrt{t/(8 \delta_{\max})}$ 
where $\delta_{\max}\!=\!\max_k \delta_k$, 
the regret of $\A$ at time $t$ is bounded by
\vspace*{-0.1cm}
\begin{align}
\lefteqn{R(\A,t) \,\le\, (c+C) \sqrt{t}\,+\,C' f(t)} 
\nonumber
\\
& 
\hspace*{0mm}
+\,  C'' t^2 
\,
\mathbb{P}\Big(
N_{k^*}(t)\!>\!\EE{ N_{k^*}(t)} + c \sqrt{ \EE{N_{k^*}(t)-1}}  
\Big)
\nonumber
\\
& 
\hspace*{0mm}
+\,  C''' t^2 
\,
\mathbb{P}\Big(
N(t)\!<\!\EE{ N(t)} - c \sqrt{ \EE{N(t)-1}}
\Big)
,
\label{eq:bnd5}
\end{align}
for some appropriate constants $C,C',C'', C'''\!>\!0$ 
that depend on the problem parameters 
$\delta_k, V_k$, the upper bound on $|X_{k,m}|$, 
and the constants $c_f$, $C_D$ and $C_N$.
\end{theorem}

\vspace{-0.1cm}

The proof of the theorem is given in \cref{sec:apx-nonuniform}.
Several comments are in order.
First,
recall that the optimal regret rate is order $\sqrt{t}$ in this setting,
up to logarithmic factors. 
To obtain such a rate, one need only achieve $f(t)\! =\! O(\sqrt{t})$,
which can be attained by stochastic or even adversarial bandit algorithms 
\citep{bubeck2012regret} 
receiving rewards with expectation $-\delta_k V_k$
and a well-concentrated number of samples.
The moment condition on $N_{k^*}(t)$ is also not restrictive; 
for example, if the estimators are rejection samplers, their sampling times 
will have a geometric distribution that satisfies the polynomial tail condition.
Furthermore, if $D_{k,m}\!\ge\!\delta^-$ 
%
%
for some $\delta^-\!>\!0$ then $N_k(t)\!<\!t/\delta^-$ for all $k$, 
which ensures the moment condition on $N_{k^*}(t)$.

Although it was sufficient to use the negative second moment $-X_{k,m}^2$
instead of variance as the bandit reward under uniform costs,
this simplification is no longer possible when costs are nonuniform,
since $\delta_k V_k\!=\! \delta_k(\E[X_{k,1}^2]-\mu^2)$ 
now involves the unknown expectation $\mu$.
Several strategies can be followed to bypass this difficulty. 
For example, given independent costs and observations,
one can use each bandit algorithm decision twice, 
feeding rewards 
$r_{k,m}\!=\!-\tfrac{1}{4}(D_{k,2m}\!+\!D_{k,2m+1})(X_{k,2m}\!-\!X_{k,2m+1})^2$
whose expectation is $\delta_k V_k$. 
Similar constructions using slightly more data can be used 
for the dependent case.

Note that ensuring $\E[T_k(N(t))]\! \le\! f(t)$ 
can be nontrivial.
Typical guarantees for UCB-type algorithms ensure that the expected 
number of pulls to a suboptimal arm $k$ in $n$ rounds is bounded by a 
function $g_k(n)$. 
However, due to their dependence, $\E[T_k(N(t))]$ cannot generally 
be bounded by $g_k(\E[N(t)])$. 
Nevertheless, if, for example, 
$D_{k,m}\geq \delta^{-}$ for some 
$\delta^{-}\!>\!0$, then $N(t)\!-\!1\! \le\! t/\delta^{-}$, 
hence $f_k(t)\! =\! g_k(t/\delta^{-})$ can be used.

Finally, we need to ensure that the last two terms in \eqref{eq:bnd5} 
remain small, 
which follows if $N(t)$ and $N_{k^*}(t)$ concentrate around their means. 
In general,
$\Prob{N\! <\! \E[N] \!-\! C \sqrt{ \E[N] \log(1/\delta)} }\le \delta$ 
for some constant $C$,
therefore $c\!=\! C \sqrt{\log(1/\delta)}$ can be chosen to achieve
$t^2 \Prob{N\!<\!\E[N]\!-\!c\sqrt{\E[N]}} \le t^2 \delta$, 
hence by chosing $\delta$ to be $O(t^{-3/2})$ we achieve
$\tilde{O}(\sqrt{t})$ regret. 
However, to ensure concentration, the allocation strategy must also 
select the optimal estimator most of the time.
For example, \citet{audibert2009exploration} show that with default 
parameters, UCB1 and UCB-V will select suboptimal arms with probability 
$\Omega(1/n)$, making $t^2 \Prob{N<\E[N]-c\sqrt{\E[N]}} = \Omega(t)$. 
However, by increasing the constant $2$ in UCB1 and the parameter 
$\zeta$ in UCB-V,
it follows from \citep{audibert2009exploration} that the chance of using 
\emph{any} suboptimal arm more than $c \log(t)\sqrt{t}$ times can be made 
smaller than $c/t$ (where $c$ is some problem-dependent constant).
%
 Outside of this small probability event, the optimal arm is used 
$t-c K \log(t) \sqrt{t}$ times,
which is sufficient to show concentration of $N(t)$.
In summary, we conclude that $\tilde{O}(\sqrt{t})$ regret can be achieved in \cref{thm:nonunifreduction}
under reasonable assumptions.

\vspace{-0.2cm}

\section{Experiments}
\label{sec:exp}

We conduct experimental investigations in a number of scenarios to
better understand the effectiveness of multi-armed bandit algorithms
for adaptive \MC estimation.


\subsection{Preliminary Investigation: A 2-Estimator Problem}

We first consider the performance of allocation strategies on a simple
2-estimator problem.  Note that this evaluation differs from standard
evaluations of stochastic bandits through the absence of
single-parameter payoff distributions, such as the Bernoulli, which
cannot have identical means yet different variances.  This is an
important detail, since stochastic bandit algorithms such as KL-UCB
and TS are often evaluated on single-parameter payoff distributions,
but their advantages in such scenarios might not extend to adaptive
\MC estimation.

\begin{figure}[t]
\centering
\includegraphics[width=\columnwidth]{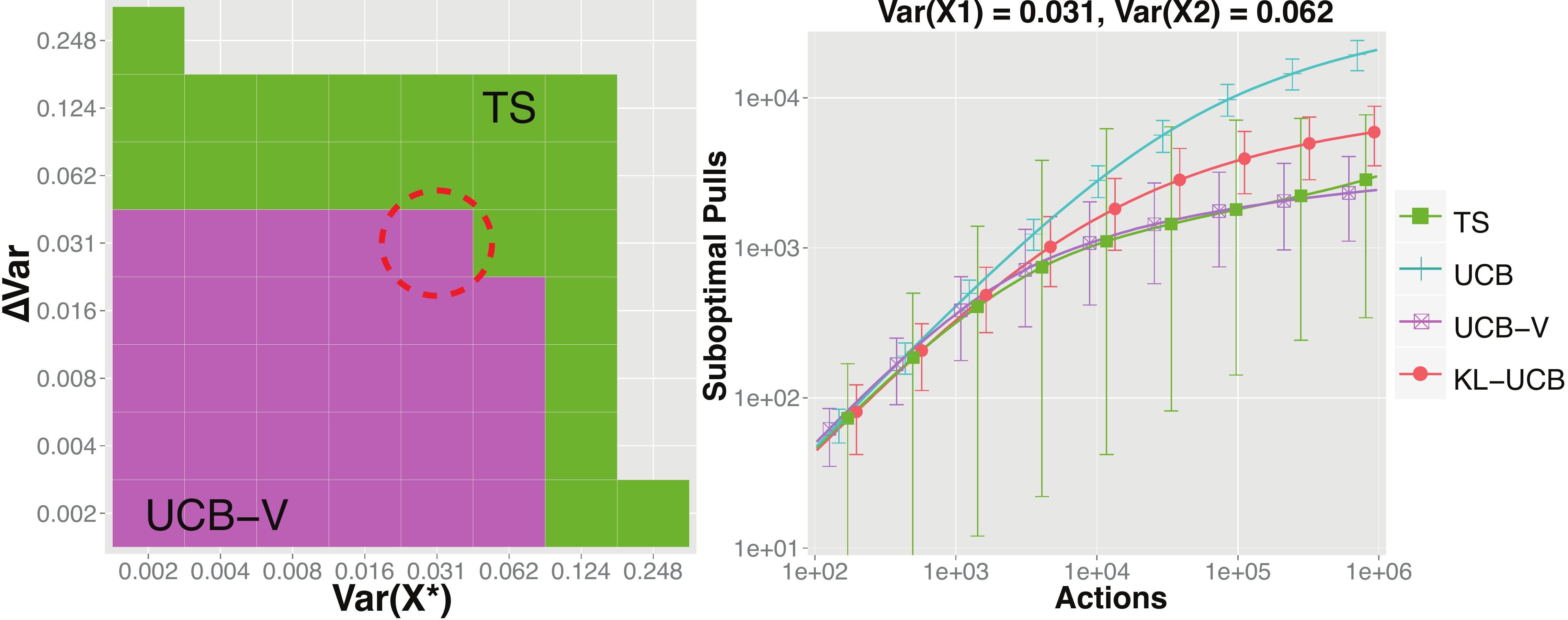}
\vspace{-0.75cm}
\caption{ 
{\bf Left}: 
Tile plot indicating which approach achieved lowest
regret (averaged over 2000 independent runs) in the 2-estimator
scaled-Bernoulli setting, at time $10^6$. 
X-axis is the variance of the optimal estimator, 
and Y-axis is the \emph{additional} variance on the second estimator. 
{\bf Right}: 
Log-plot illustrating the expected number of suboptimal selections
for the highlighted case (dashed red circle). 
Error bars indicate 99\% empirical percentiles.
}
\vspace{-0.15cm}
\label{fig:grid1}
\end{figure}

\begin{figure*}[ht]
\centering
\includegraphics[width=0.99\textwidth]{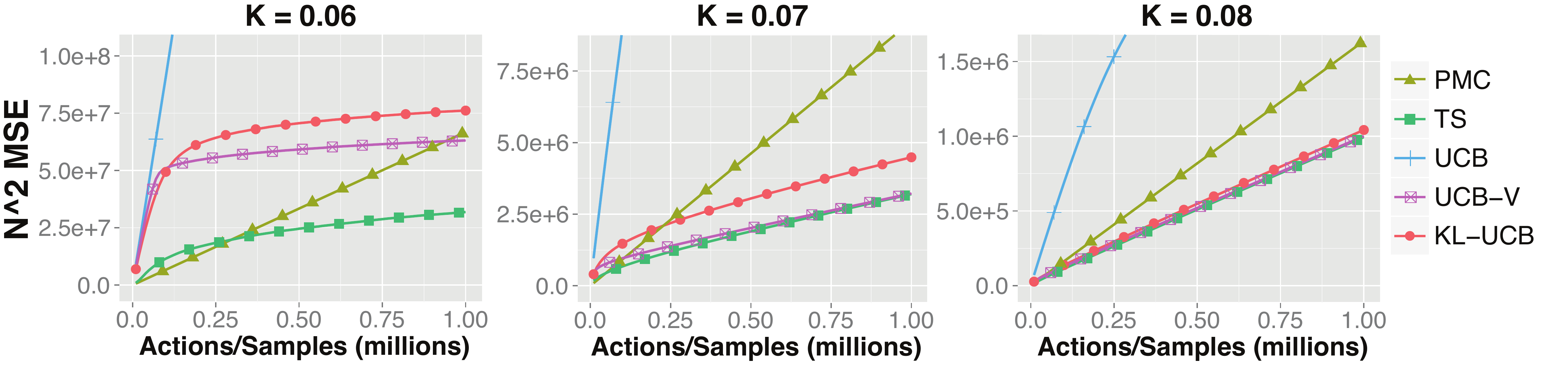}
\vspace{-0.4cm}
\caption{
Plots showing the normalized MSE of different adaptive strategies 
for estimating the price of a European caplet option under the
Cox-Ingersol-Ross interest rate model, using different strike prices (K). 
All results are statistically significant up to visual resolution. 
}
\vspace{-0.34cm}
\label{fig:cir}
\end{figure*}

In particular, we consider problems when $X_{k,t}=\mu + s_k(Z_{t} -
\frac{1}{2})$, where $Z_t$ is standard Bernoulli and $s_k\! \in\!
(0,1)$ is a separate scale parameter for $k\!\in\!\{1,2\}$.  This
design permits the maximum range for variance around a mean within a
bounded interval.  We evaluated the four bandit strategies detailed in
\cref{sec:background}: UCB1, UCB-V, KL-UCB, and TS, where for UCB-V we
used the same settings as \citep{audibert2009exploration}, and for TS
we used the uniform Beta prior, i.e., $\alpha_0 = 1$ and $\beta_0 = 1$.

The relative performance of these approaches is reported in
\Cref{fig:grid1}.  TS appears best suited for scenarios where
\emph{either estimator} has high variance, whereas UCB-V is more
effective when faced with medium or low variance estimators.
Additionally, KL-UCB out-performs UCB-V in high variance settings, but
in all such cases was eclipsed by TS.

\vspace{-0.2cm}
\subsection{Option Pricing}

We next consider a more practical application of adaptive \MC 
estimation to the problem of pricing financial instruments.
In particular, following \citep{douc2005,arouna2004}, 
we consider the problem of pricing \emph{European call options} 
under the assumption that the interest rate evolves in time according 
to the Cox-Ingersoll-Ross (CIR) model \cite{cox1985theory}, 
a popular model in mathematical finance
(details provided in \cref{sec:appendix-finance-example}).
In a nutshell, this model assumes that the interest rate $r(t)$, 
as a function of time $t\!>\!0$,
follows a \emph{square root diffusion model}.
The price of a European caplet option with ``strike price'' $K\!>\!0$,
``nominee amount'' $M\!>\!0$ and ``maturity'' $T\!>\!0$ 
is then given by 
$P = M\exp(-\int_0^T r(t)dt)\max(r(T) - K, 0)$.
The problem is to determine the expected value of $P$.

A naive approach to estimating $\EE{P}$ is to simulate independent
realizations of $r(t)$ for $0<t\leq T$.  However, any simulation where
the interest rate $r(T)$ lands below the strike price can be ignored
since the payoff is zero.  Therefore, a common estimation strategy is
to use importance sampling by introducing a ``drift'' parameter
$\theta>0$ into the proposal density for $r(t)$, with $\theta=0$
meaning no drift; this encourages simulations with higher interest
rates.  The importance weights for these simulations can then be
efficiently calculated as a function of $\theta$ (see
\cref{sec:appendix-finance-example}).


Importantly, the task of adaptively allocating trials between
different importance sampling estimators has been previously studied
on this problem, using an unrelated technique known as \emph{d-kernel
  Population Monte Carlo~(PMC)}~\cite{douc2005}.  Space restrictions
prevent us from providing a full description of the PMC method, but,
roughly speaking, the method defines the proposal density as mixture
over the set $\{\theta_k\}$ of drift parameters considered.  At each
time step, PMC samples a new drift value according to this mixture and
then simulates an interest rate.  After a fixed number of samples, say
$G$ (the \emph{population} size), the mixture coefficient $\alpha_k$
of each drift parameter $\theta_k$ is adjusted by setting it to be
proportional to the sum of importance weights sampled from that
parameter: $\alpha_k = \frac{\sum_{t=1}^G w_t\indicate{I_t =
    k}}{\sum_{t=1}^G w_t}$.  The new proposal is then used to generate
the next population.

We approximated the option prices under the same parameter settings as
\citep{douc2005},
namely, $\nu=0.016$, $\kappa\!=\!0.2$, $r_0 \!=\! 0.08$,
$T=1$, $M\!=\!1000$, $\sigma \!=\! 0.02$, and $n \!=\! 100$, for
different strike prices $K\!=\!\{ 0.06, 0.07, 0.08\}$
(see \cref{sec:appendix-finance-example}).
However, we consider a wider set of proposals
given by $\theta_k = k/10$ for $k\!\in\!\{0,1,...,15\}$.  
The results averaged over
1000 simulations are given in \Cref{fig:cir}.

These results generally indicate that the more effective bandit
approaches are significantly better suited to this allocation task
than the PMC approach, particularly in the longer term.  Among the
bandit based strategies, TS is the clear winner, which, given the
conclusions from the previous experiment, is likely due to high level
of variance introduced by the option pricing formula.  Despite this
strong showing for the bandit methods, PMC remains surprisingly
competitive at this task, doing uniformly better than UCB, and better
than all other bandit allocation strategies early on for $K=0.06$.
However, we believe that this advantage of PMC stems from the fact that
 PMC explore the
entire space of mixture distribution (rather than  single $\theta_k$).  
It remains an interesting area for future work in bandit-based allocation
strategies to extend the existing methods to continuously parameterized
settings.


\vspace{-0.2cm}
\subsection{Adaptive Annealed Importance Sampling}

\begin{figure*}[t]
\centering
\includegraphics[width=0.99\textwidth]{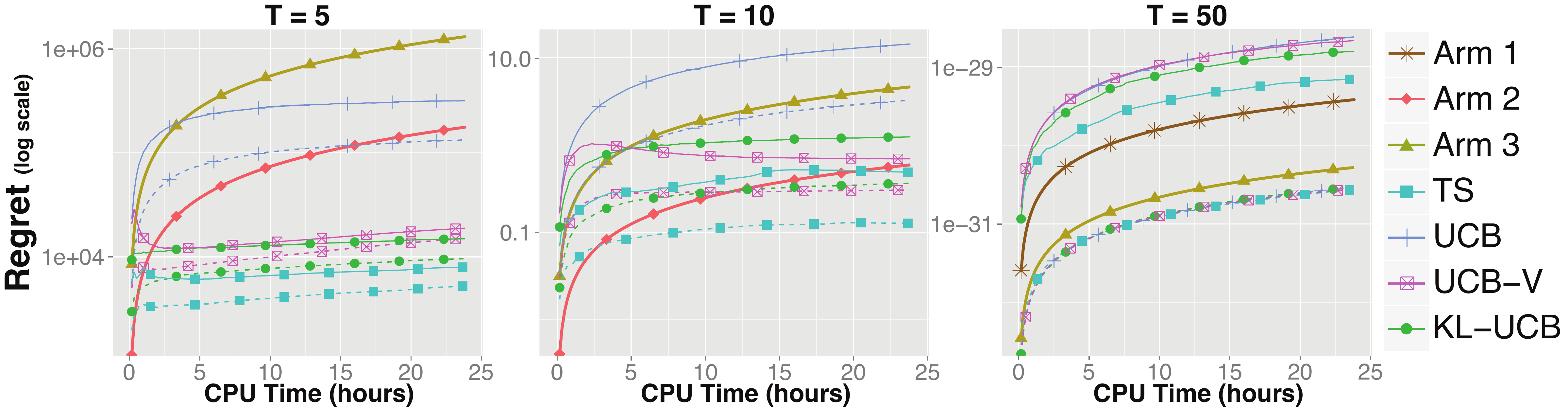}
\vspace{-0.4cm}
\caption{
Plots showing the average regret (over 20 independent runs)
of bandit allocators on the Logistic Regression model.
$T$ is training sample size, ``Arm 1/2/3'' indicates each fixed estimator;
and the one missing from a figure is the best for that setting.
The solid lines indicate the performance of combining observations uniformly,
whereas the dashed lines indicate the performance of combining observations
using \emph{inverse variance weights}.
}
\vspace{-0.34cm}
\label{fig:lr}
\end{figure*}

Many important applications of \MC estimation occur in Bayesian inference,
where a particularly challenging problem is evaluating the
\emph{model evidence} of a latent variable model.
Evaluating such quantities is useful for a variety of purposes,
such as Bayesian model comparison and testing/training set evaluation
\citep{robert2012bayesian}.  
However, the desired quantities are notoriously difficult to estimate
in many important settings,
due in no small part to the fact that popular high-dimensional
Monte Carlo strategies, such as Markov Chain Monte Carlo~(MCMC) methods, 
cannot be directly applied \citep{Neal2005}.

Nevertheless,
a popular approach for approximating such values is 
\emph{annealed importance sampling}~(AIS) \citep{Neal2001} (or more
generally sequential Monte Carlo samplers \citep{del2006}).  In a
nutshell, AIS combines the advantages of importance sampling with MCMC
by defining a proposal density through a sequence of MCMC transitions
applied to a sequence of \emph{annealed} distributions, which slowly
blend between the proposal (prior) and the target (un-normalized
posterior).
While such a technique can offer impressive practical advantages, 
it often requires considerable effort to set parameters;
in particular, the practitioner must specify the number of annealing steps, 
the annealing rate or ``schedule'', the underlying MCMC method 
(and its parameters), and the number of MCMC transitions to execute at
annealing step.  
Even when these parameters have been appropriately tuned on preliminary data,
there is no assurance that these choices will remain effective when 
deployed on larger or slightly different data sets.

Here we consider the problem of approximating the normalization
constant for a Bayesian logistic regression model on different sized
subsets of the 8-dimensional \emph{Pima Indian diabetes} UCI data
set~\cite{Bache+Lichman:2013}.  
We consider the problem of allocating resources between three AIS estimators 
that differ only in the number of annealing steps they use;
namely, 400, 2000, and 8000 steps.
In each case, we fix the annealing schedule using the \emph{power of 4}
heuristic suggested by \citep{kuss2005}, with a single slice sampling
MCMC transition 
used at each step \citep{neal2003slice} 
(this entails one ``step'' in each of the 8 dimensions);
see \cref{sec:appendix-lr-example} for further details.



A key challenge in this scenario is that the computational costs associated 
with each arm differ substantially, and, because slice sampling uses an
internal rejection sampler, these costs are stochastic.  
To account for these costs we directly use \emph{elapsed CPU-time}
when drawing a sample from each estimator, as reported by the 
\textsc{java vm}.  
This choice reflects the true underlying cost and is particularly convenient
since it does not require the practitioner to implement special 
accounting functionality.
Since we do not expect this cost to correlate with the sample returns, 
we use the independent costs payoff formulation from \cref{sec:nonuniformcost}:
$-\tfrac{1}{4}(D_{k,2m}+D_{k,2m+1})(X_{k,2m}-X_{k,2m+1})^2$.

The results for the different allocation strategies for training sets
of size 5, 10, and 50 are shown in \Cref{fig:lr}.  
Perhaps the most striking result is the performance improvement achieved 
by the \emph{nonuniformly combined estimators},
which are indicated by the dashed lines.  
These estimators do not change the underlying allocation;
instead they improve the final combined estimate by weighting each observation
inversely proportional to the sample variance of the estimator that produced it.
This performance improvement is an artifact of the
nonuniform cost setting, since arms that are very close in terms of
$V_k\delta_k$ can still have considerably different variances, 
which is especially true for AIS.
Also observe that no one arm is optimal for all three training set sizes, 
consequently, we can see that bandit allocation
(Thompson sampling in particular) is able to outperform any static strategy.  
In practice, this implies that even after exhaustive
parameter tuning, automated adaptation can still offer considerable benefits
simply due to changes in the data.

\vspace*{-0.3cm}

\section{Conclusion}

In this paper we have introduced a new sequential decision making strategy
for competing with the best consistent \MC estimator in a finite pool.
When each base estimator produces unbiased values at the same cost, 
we have shown that the sequential estimation problem maps to a 
corresponding bandit problem, allowing
 future improvements in bandit algorithms to be transfered
 to combining unbiased estimators.
We have also shown a weaker reduction for problems where the different 
estimators take different (possibly random) time to produce an observation.

We expect this work to inspire further research in the area. 
For example, one may consider combining not only finitely many, 
but infinitely many estimators using appropriate bandit techniques
\citep{bubeck2011}, and/or exploit the fact that the observation 
from one estimator may reveal information about the variance of others.
This is the case for example when the samplers use importance sampling, 
leading to the (new) stochastic variant of the problem known as 
``bandits with side-observations'' \citep{MaSha11,AlCBGeMa13}.
However, much work remains to be done,
such studying in detail the use variance weighted estimators,
dealing with continuous families of estimators,
or a more thorough empirical investigation of the alternatives
available.

\subsection*{Acknowledgements}
{\footnotesize
This work was supported by the Alberta Innovates Technology
Futures and NSERC. 
Part of this work was done while Cs.Sz. was visiting
 Technion, Haifa and Microsoft Research, Redmond,
whose support and hospitality are greatly acknowledged.
}

\newpage 

\footnotesize

\bibliographystyle{icml2014} 
\bibliography{paper}

\normalsize

\newpage
\onecolumn
\appendix

\section{Proofs for \cref{sec:uniformcost}}
\label{sec:apx-uniform}
In this section we provide the proof of \cref{thm:regretmatch}.
For the proof, we will need the following two lemmas:

\begin{lemma}[Optional Sampling]
\label{lem:i.i.d.SequentialChoice}
Let $(X_{t})_{t \in \N}$ be a sequence of i.i.d. random variables, and $(X'_{t})_{t \in \N}$ be its subsequence such that the decision whether to include $X_{t}$ in the subsequence is independent of future values in the sequence, i.e., $X_{s}$ for $s \ge t$. Then the sequence $(X'_{t})_{t \in \N}$ is an i.i.d. sequence with the same distribution as $(X_{t})_{t \in \N}$.
\end{lemma}
\begin{proof}
See Theorem~5.2 in Chapter III on page 145 of \citep{Doob:1953}.
\end{proof}
We also need Wald's second identity:
\begin{lemma}[Wald's Second Identity]
\label{lem:wald2}
Let $(X_{t})_{t \in \N}$ be a sequence of $(\FF_t; t\ge 1)$-adapted random variables
such that $\EE{ X_t | \FF_{t-1} } = \mu$ and $\EE{ (X_t - \mu)^2 | \FF_{t-1} } = \sigma^2$ for any $t\ge 1$.
Let $S_n = \sum_{t=1}^n X_t$ be the partial sum of the first $n$
random variables ($n\ge 1$) and $\tau>0$ be some stopping time
w.r.t. $(\FF_t; t\ge 1)$.%
\footnote{$\tau$ is called a stopping time w.r.t. $(\FF_t; t\ge 1)$ if $\{\tau \le t\} \in \FF_t$ for all $t \ge 1$.}
Then,
\[
\EE{ (S_{\tau} - \mu \tau)^2 } = \sigma^2  \EE{ \tau}\,.
\]
\end{lemma}
\begin{proof}
See Theorem~14.3 of \citep{Gut05}.
\end{proof}

Now, let us return to the proof of \cref{thm:regretmatch}.
Let $I_t\in \{1,\ldots,K\}$ denote the choice that $\A$ made at the beginning of round $t$ and
recall that $T_k(t) = \sum_{s=1}^t \indicate{I_s = k}$.
The observation at the end of round $t$ is $Y_t = X_{I_t,T_{I_t}(t)}$ 
and the cumulative sum of returns after $n$ rounds for arm $k$ is
$S_{k,n} = \sum_{t=1}^{T_k(n)} X_{k,t}$.
Likewise, $S_n = \sum_{t=1}^n Y_t = \sum_{k=1}^K S_{k,n}$.  
By definition, the estimate after $n$ rounds is
\[
\hmu_n = \frac{S_n}{n} 
= \frac{1}{n} \sum_{k=1}^K S_{k,n}\,.
\]
Since the variance of this estimate does not depend on the mean,
without loss of generality we may consider the case when $\mu = 0$,
giving 
\begin{align}
L_n(\A) & = \EE{ \hmu_n^2 }
      = \frac{1}{n^2}\E\left[\sum_{k=1}^K S_{k,n}^2 +
    2\sum_{k \ne j} S_{k,n}S_{j,n}\right]
\label{eqn:SnMSE}
\end{align}
We may then observe that for any $k \neq j$,
\begin{align*}
\E\left[S_{k,n}S_{j,n}\right]  
&= \E\left[(\indicate{I_n = k}Y_n + S_{k,n-1})
(\indicate{I_n = j}Y_n + S_{j,n-1})\right] \\
& = \E[\cancel{\indicate{I_n = k}\indicate{I_n = j}Y^2_n} + \indicate{I_n =
  k}Y_nS_{j,n-1} 
 + \indicate{I_n = j}Y_nS_{k,n-1} + S_{k,n-1}S_{j,n-1}]
.
\end{align*}
By considering the conditional expectation w.r.t.  the history
 up to $Y_n$ (including $I_n$), i.e., w.r.t. $\mathcal{F}_{n-1}
= \sigma(Y_1,\ldots,Y_{n-1},I_1,\ldots,I_{n})$, we have
\begin{align*}
  \E[S_{k,n}S_{j,n}] &=\EE{ \EE{ S_{k,n}S_{j,n}|\F_{n-1}}}\\
& = \E[(\indicate{I_n = k}S_{j,n-1}
+\indicate{I_n = j}S_{k,n-1})\,\E[Y_n | \mathcal{F}_{n-1}]] 
+ \E[ S_{k,n-1} S_{j,n-1}]
.
\end{align*}
Now, since $I_n\in \{1,\ldots,K\}$ and $I_n$ is $\F_{n-1}$]-measurable,
 $\E\left[Y_n | \mathcal{F}_{n-1}\right] = \sum_{k=1}^K \indicate{I_n =  k} \EE{X_{k,T_k(n)}|\F_{n-1},I_n=k}$.
Further, by \cref{lem:i.i.d.SequentialChoice}, $(X_{k,T_k(n)})_n$ is an i.i.d. sequence,
  sharing the same distribution as $(X_{k,t})_t$.
Since $\sigma(X_{k,T_k(n)})$ is independent of $\sigma(Y_1,\ldots,Y_{n-1},I_1,\ldots,I_{n-1},I_n,I_{n}=k)$,
	$\EE{X_{k,T_k(n)}|\F_{n-1},I_n=k} = \EE{X_{k,T_k(n)}} = \mu = 0$.
Therefore,
$$
\E\left[S_{k,n}S_{j,n}\right] = \E[S_{k,n-1}S_{j,n-1}] = \ldots = \E[S_{k,0} 
S_{j,0}] =0
.
$$
As a result,
\eqref{eqn:SnMSE} becomes 
\[
L_n(\A) = \frac{1}{n^2}\E\left[\sum_{k=1}^K
    S_{k,n}^2\right] = \frac{1}{n^2}\sum_{k=1}^K \E[S_{k,n}^2]
.
\]
From Wald's second identity (cf. \cref{lem:wald2})  we conclude $\E[S^2_{k,n}] =
\var(X_k)\E[T_k(n)]$ and thus get
\begin{equation}
\label{eq:wald2}
L_n(\A) = \frac{1}{n^2} \sum_{k=1}^K \var(X_k)\E[T_k(n)]
.
\end{equation}
Using the definition of the normalized MSE-regret and that $\min_{1\le k \le K} L_{k,n} = V^*/n$,
\begin{align*}
R_n(\mathcal{A}) = n^2\left(L_n(\A)- \frac{V^*}{n}
  \right) = \sum_{k=1}^K \E[T_k(n)](V_k - V^*) 
,
\end{align*}
which was the result to be proven.

\section{Handling Unknown Ranges}
\label{sec:apx-range}

The algorithms and bounds of  \cref{sec:uniformcost}
can be easily extended to the setting where 
$X_{k,t}\in [a_k,b_k]$  (often $a=0$) where $a_k<b_k$ are \emph{a priori} known.
One option is to scale $X_{k,t}$ to the common range $[0,1]$ using 
$\tilde{X}_{k,t} = (X_{k,t} - \min_k a_k)/(\max_k b_k - \min_k a_k)$ 
and then feed $1-\tilde{X}_{k,t}^2$ to the bandit algorithms 
(as the bandit algorithms expect the rewards in $[0,1]$ and constant 
translation of each reward does not change the regret).
However, as these algorithms (prepared for ``worst case'' distributions) 
are sensitive to the overestimation of the range, this would lead to an 
unnecessary deterioration of the performance. A better option is to scale 
each variable separately.
Then, the upper-confidence bound based algorithms must be modified by scaling 
each of the rewards with respect to its own range and then the bounds needs to 
be scaled back to the original range. 
Thus, the method that computes the reward upper bounds must be fed with 
$\frac{ (b_k-a_{\min})^2 - (X_{k,m}-a_{\min})^2}{ (b_k-a_{\min})^2 }\in [0,1]$, 
where $a_{\min} = \min_{1\le k \le K} a_k$.
Then, if the method returns an upper bound $B_{k,m}$, the bound to be used 
for selecting the best arm (which should be an upper bound for 
$-\EE{(X_{k,t}-a_{\min})^2}$) should be 
$B'_{k,t} = (b_k-a_{\min})^2 (B_{k,t}-1)$.
Here we exploit that $-\EE{(X_{k,t}-a_{\min})^2} = -\EE{X_{k,t}^2}+c$ 
where the constant $c = 2\mu a_{\min}-a_{\min}^2$ 
(which is neither known, nor needs to be known or estimated)
is common to all the estimators, hence finding the arm that maximizes 
$-\EE{(X_{k,t}-a_{\min})^2}$ is equivalent to finding the arm that 
maximizes $-\EE{ X_{k,t}^2}$.
Since the confidence bounds will be loser if the range is larger, 
we thus see that the algorithm may be sensitive to the ranges 
$r_k = b_k - a_{\min}$. 
In particular, the $1/n$-term in the bound used by UCB-V will scale with 
$r_k^2$ and may dominate the bound for smaller values of $n$.
In fact, UCB-V needs $n \approx r_k^2$ samples before this term 
becomes negligible, which means that for such estimators, 
the upper confidence bound of UCB-V will be above $1$ for $r_k^2$ steps.
Even if we cut the upper confidence bounds at $1$, 
UCB-V will keep using these arms with a large range $r_k$
defending against the possibility that the sample variance crudely 
underestimates the true variance.
Since the bound of KL-UCB 
is in the range $[0,1]$, KL-UCB is less exposed to this problem.

Based on \cref{thm:regretmatch}, an alternative is to use bandit algorithms 
to minimize cost where the cost of an arm is defined as the variance of 
samples from that arm.
This may be advantageous when the ranges become large because of unequal 
lower bounds $(a_k)$.
To implement this idea one needs to develop new bandit algorithms for 
cumulative variance minimization,
which in fact is explored in \cref{sec:nonuniformcost},
in the context of estimation under nonuniform-costs.

Finally, note that the assumption that the samples belong to a known bounded 
interval is not necessary (it was not used in the reduction results). 
In fact, the upper-confidence based bandit algorithms mentioned can also be
applied when the payoffs are
subgaussian with a known subgaussian coefficient,%
\footnote{
A centered random variable $X$ is subgaussian if 
$\Prob{|X|\ge t} \le c \exp(-t^2/(2\sigma^2))$ for all real $t$ with some $c,\sigma>0$.
}
or even if the tail is heavier \citep{BuCB12:FTML,BuCBLu13}.
In fact, the weaker assumptions under which the multi-armed bandit problem 
with finitely many arms has been analyzed assumes only that the $1+\eps$ 
moment of the payoff is finite for some known $\eps$ with a known moment 
bound \citep{BuCBLu13}.
These strategies must replace the sample means with more robust estimators 
and also modify the way the upper confidence bounds are calculated. 
In our setting, the condition on the moment transfers to the assumption that
the $2+\eps$ moment $\EE{ |X_{k,t}|^{2+\eps}}$ 
must be finite with a known upper bound for some $\eps$.

\section{Proofs for \cref{sec:nonuniformcost}}
\label{sec:apx-nonuniform}
In this section we provide the proof of \cref{thm:nonunifreduction}.

We can write the MSE at time $t$ as $L(\A,t) = \EE{ \left(\frac{S(t)-(N(t)-1)\mu}{N(t)-1}\right)^2 }$.
The first step is to replace the denominator with its expectation.
The price of this is calculated in the following result:
\begin{lemma}
\label{lem:SNconc}
Let $S,N$ be random variables, $N\ge 0$ such that $\EE{N}>0$, 
$\mu\in \real$. Let $D = \ind{N=0}(a-\mu) + \ind{N\ne 0} \left(\frac{S}{N}-\mu\right)$.
Let $d\in \R$ be an almost sure upper bound on $|D|\one{N\ne 0}$.
Then, for any $c\ge 0$,
\begin{align}
\label{lem:SNconcup}
\EE{ D^2 }
& \le 
(a-\mu)^2 \Prob{N=0} 
 + \frac{\EE{(S- N\mu)^2}}{\EE{N}^2} \left(1+\frac{2c}{\sqrt{\EE{N}}}\right)^2 
 + d^2 \left[ \Prob{N<\EE{N}-c\sqrt{\EE{N}}} + \one{\EE{N}<4c^2}\right]\,.
\end{align}
Further, assuming that $S=0$ almost surely on $\{ N=0 \}$,
\begin{align*}
\EE{ D^2 }
&\ge
(a-\mu)^2 \Prob{N=0} 
+ \frac{\EE{(S-N\mu)^2}}{\EE{N}^2}\left(1-\frac{2c}{\sqrt{\EE{N}}}\right) 
 - d^2\,\frac{ \EE{N^2}}{\EE{N}^2} \, \Prob{N>\EE{N}+c\sqrt{\EE{N}}} \,.
\end{align*}
\end{lemma}
The lemma is useful when $N$ concentrates around its mean. 
Consider the upper bound on $\EE{D^2}$.
In general, we  expect 
$\Prob{ N < \EE{N} - C \sqrt{ \EE{ N} \log(1/\delta)} }\le \delta$ with some numerical constant $C$. 
Thus, setting $c= C \sqrt{\log(1/\delta)}$ results in 
$d^2 \left[ \Prob{N<\EE{N}-c\sqrt{\EE{N}}} + \one{\EE{N}<4c^2}\right]
\le d^2 \delta + d^2 \one{\EE{N}<4C^2 \log(1/\delta)}$.
Imagine now that $\EE{N} = r t $ and choose $\delta = t^{-q}$ with $q>0$.
Then, as soon as $r t> 4 C^2 \sqrt{ q \log(t)}$, the last indicator of~\eqref{lem:SNconcup} becomes zero. 
Further, $\frac{2c}{\sqrt{\EE{N}}} = \sqrt{\frac{2 q \log(t)}{r t}} \ra 0$ as $t\ra \infty$.
Since this term converges at a $1/\sqrt{t}$ rate, we should choose $q \ge 1/2$, so that $\delta = \delta(t)$ converges at least as fast as this term. However, since later ``$D$ gets multiplied by $t^2$'', to get a sublinear rate matching other terms, one should choose $q\ge 3/2$. 
\begin{proof}
Let $b = c\sqrt{\EE{N}}$. First we prove the upper bound.
We have
\begin{align*}
\EE{ D^2 } 
& \le 
\EE{ D^2 \one{N=0} } 
+  \EE{ D^2 \one{N\ne 0, N\ge \EE{N}-b, b\le \tfrac12\EE{N}} } 
+  \EE{ D^2 \one{N\ne 0, N< \EE{N}-b } } \\
& \qquad + \EE{ D^2 \one{N\ne 0, b> \tfrac12\EE{N}} }\\
& \le 
(a-\mu)^2 \Prob{N=0} 
+ \EE{ \frac{(S-N\mu)^2}{(\EE{N}-b)^2} \one{ b\le \tfrac12\EE{N}} } 
+ d^2 \Prob{ N< \EE{N}-b } 
+ d^2 \one{ b> \tfrac12\EE{N}}\\  
& \le  
(a-\mu)^2 \Prob{N=0} 
+  \frac{\EE{(S-N\mu)^2}}{(\EE{N}-b)^2} \one{ b\le \tfrac12\EE{N}} 
+ d^2  \Prob{ N< \EE{N}-b } 
+ d^2 \one{ b> \tfrac12\EE{N}}\,.
\end{align*}
Noting that $1/(1-x)\le 1+2x$ when $0\le x\le 1/2$, we get that
\begin{align*}
\frac{1}{\EE{N}-b} \one{ b\le \tfrac12\EE{N}} 
 =\frac{1}{\EE{N}} \,\frac{1}{1 - \frac{b}{\EE{N}}} \one{ b\le \tfrac12\EE{N}} 
 \le\frac{1}{\EE{N}} \,\left( 1+ 2 \frac{b}{\EE{N}}\right) = \frac{1}{\EE{N}} \left(1 + 2 \frac{c}{\sqrt{\EE{N}}}\right)\,.
\end{align*}
Putting things together, we get the desired upper bound. 

The lower bound is proved in a similar fashion.
\newcommand{\Nb}{\mathcal{N}}
To simplify notation, let $\Nb$ denote the event $\{0 < N \le \EE{N}+b\}$. Then,
\begin{align*}
\EE{ D^2 } 
& \ge 
(a-\mu)^2 \Prob{N=0} 
+ \EE{\left(\frac{(S-N\mu)}{N}\right)^2  \ind{ \Nb} }   \\
& \ge 
(a-\mu)^2 \Prob{N=0} 
+ \EE{\frac{(S-N\mu)^2}{(\EE{N}+b)^2}  \ind{ \Nb} }  \\
& =
(a-\mu)^2 \Prob{N=0} 
+ \EE{\frac{(S-N\mu)^2}{(\EE{N}+b)^2}  }  - \EE{ \frac{(S-N\mu)^2}{(\EE{N}+b)^2} \ind{ N=0 \text{ or } N>\EE{N}+b}} \\
& \ge 
(a-\mu)^2 \Prob{N=0}
 + \frac{\EE{N}^2}{(\EE{N}+b)^2}\frac{\EE{(S-N\mu)^2}}{\EE{N}^2} 
 -  \EE{ \frac{(S-N\mu)^2}{(\EE{N}+b)^2} \ind{  N>\EE{N}+b}}\\
& \qquad \qquad (\text{because by assumption } \EE{ (S-N\mu)^2 \ind{ N=0 }} =0)\\ 
& =  
(a-\mu)^2 \Prob{N=0}
 + \frac{\EE{N}^2}{(\EE{N}+b)^2}\frac{\EE{(S-N\mu)^2}}{\EE{N}^2} 
 - \frac{\EE{N^2}}{(\EE{N}+b)^2} \left(d^2 \Prob{N>\EE{N}+b}  \right) \\
& \ge
(a-\mu)^2 \Prob{N=0} 
 + \left(1-\frac{2b}{\EE{N}}\right)\frac{\EE{(S-N\mu)^2}}{\EE{N}^2} 
- d^2\, \frac{ \EE{N^2} \,\Prob{N>\EE{N}+b} }{\EE{N}^2}\,,
\end{align*}
where in the last inequality we used
\[
\frac{\EE{N}^2}{(\EE{N}+b)^2} = \left(1-\frac{b}{\EE{N}+b}\right)^2 \ge 1-\frac{2b}{\EE{N}+b}
\]
and $1/(\EE{N}+b) \le 1/\EE{N}$. 
\end{proof}

The above lemma suggests that 
when bounding $\EE{(S/N-\mu)^2}$,
the main term  is $\EE{(S-N\mu)^2}/\EE{N}^2$.
First we develop a lower bound on this when only sampler $k$ is used at every time step.
\begin{lemma}
\label{lem:SEN}
Let~\Cref{ass:cost} hold and assume that the random variables $|X_{k,m}-\mu|_{k,m}$  are a.s. bounded by some constant $B>0$.
Consider the case when an individual sampler $k \in \{1,\ldots,K\}$ is used up to time $t$:
Let $N_k(t)$ be the smallest integer such that $\sum_{m=1}^{N_k(t)} D_{k,m} \ge t$ and $S_k(t)=\sum_{m=1}^{N_k(t)-1} X_{k,m}$.
Then, for any $t>0$ such that $\EE{N_k(t)} > 1$, and for any constant $\beta>0$ such that $\Prob{D_{k,1} \le \beta}>0$,
\begin{align}
\label{eq:vkbound}
\frac{\EE{(S_k(t)-(N_k(t)-1)\mu)^2}}{\EE{N_k(t)-1}^2} 
\ge \frac{\delta_k V_k}{t+\beta}
-  2B\frac{\delta_k \sqrt{V_k \delta_k}}{t^{3/2}} 
- \frac{V_k \EE{D_{k,1} \ind{D_{k,1}>\beta}}}{t}\,
\end{align}
and
\begin{align}
\label{eq:nkbound}
\frac{t}{\delta_k} \le \EE{N_k(t)} \le \frac{t+\beta}{\delta_k - \EE{D_{k,1} \ind{D_{k,1}>\beta}}},
\end{align}
\end{lemma}
If the random variables $(D_{k,m})$ are a.s. bounded by a constant, we can choose $\beta$ to be their common upper bound to cancel the third term of~\eqref{eq:vkbound}.
Otherwise, we need to select $\beta$ to strike a good balance between the first and third terms.
Note that $\beta = o(t)$ makes the first the same order as $\delta_k V_k/t$ with an error term of order $\beta/t^2$.
Thus, a reasonable choice is $\beta = t^{3/2}$, which makes the error term of the same order as the second term of the lower bound.
\newcommand{\oT}{\overline{T}}
\begin{proof}
We start with the proving the upper bound on $\EE{N_k(t)}$ in~\eqref{eq:nkbound}.
Define $\hat{D}_{k,m}=\min\{D_{k,m}, \beta\}$ and let
$\hat{N}_k(t)$ be the smallest integer such that $\sum_{m=1}^{\hat{N}_k(t)} \hat{D}_{k,m} \ge t$. 
Then, clearly $\hat{N}_{k}(t) \ge N_{k}(t)$, and by Wald's identity, 
\begin{align*}
t & > \EE{\sum_{m=1}^{\hat{N}_k(t)-1} \hat{D}_{k,m}} 
\ge \EE{\sum_{m=1}^{\hat{N}_k(t)} \hat{D}_{k,m}} - \beta \\
& = \EE{\hat{N}_k(t)} \EE{\hat{D}_{k,1}} -\beta
\ge \EE{N_k(t)}(\delta_k - \EE{D_{k,1} \ind{D_{k,1}>\beta}}) - \beta.
\end{align*}
Therefore, 
\[
\EE{N_k(t)} \le \frac{t+\beta}{\delta_k - \EE{D_{k,1} \ind{D_{k,1}>\beta}}},
\]
where we used that $\delta_k - \EE{D_{k,1} \ind{D_{k,1}>\beta}}>0$, which follows 
from our assumption $\Prob{D_{k,1} \le \beta}>0$, finishing the proof of the upper bound.
As for the lower bound, by the definition of $N_k(t)$ 
and Wald's identity we have $\EE{N_k(t)} \delta_k \ge t$, 
 thus finishing the proof of~\eqref{eq:nkbound}.

Let us now turn to proving~\eqref{eq:vkbound}.
We would like to apply Wald's second identity to $ (S_k(t) - (N_k(t)-1)\mu)^2$.
However, while $N_k(t)$ is a stopping time w.r.t. the (natural) filtration 
	$\FF_m = \sigma( X_{k,1},D_{k,1},\ldots,X_{k,m},D_{k,m} )$, 
	$N_k(t)-1$ is not.
Define $F = \sum_{m=1}^{N_k(t)} X_{k,m} - N_k(t) \mu$: We can apply Wald's identity to $F$.
Since $(S_k(t) - (N_k(t)-1)\mu)^2 = F-(X_{k,N_k(t)}-\mu)$ we get
\begin{align}
\EE{ (S_k(t) - (N_k(t)-1)\mu)^2  }
 &= \EE{ F^2 } + \EE{ (X_{k,N_k(t)} - \mu)^2 } - 2 \EE{ (X_{k,N_k(t)} - \mu) F}\nonumber \\
 & \ge \EE{ N_k(t) }  V_k - 2 \EE{ (X_{k,N_k(t)} - \mu) F} \nonumber\\
 & \ge \EE{ N_k(t) }  V_k -2 B (\EE{F^2})^{1/2} \nonumber \\
 & = \EE{ N_k(t) }  V_k  - 2 B \sqrt{ \EE{N_k(t)} V_k }\,, \label{eq:wald2a0}
\end{align}
where in the second and last equality we used \cref{lem:wald2}.

Combining \eqref{eq:nkbound}  with \eqref{eq:wald2a0}, we obtain
\begin{align*}
\frac{\EE{(S_k(t)-(N_k(t)-1)\mu)^2}}{\EE{N_k(t)-1}^2} 
&\ge \frac{\EE{N_k(t)} V_k}{\EE{N_k(t)}^2}
- \frac{2B \sqrt{ \EE{N_k(t)} V_k}}{\EE{N_k(t)}^2}\\
&= \frac{V_k}{\EE{N_k(t)}}
- \frac{2B \sqrt{ \EE{N_k(t)} V_k}}{\EE{N_k(t)}^2}\\
& \ge \frac{\delta_k V_k}{t+\beta}
- \frac{V_k \EE{D_{k,1} \ind{D_{k,1}>\beta}}}{t+\beta}
- \frac{2B \sqrt{ \EE{N_k(t)} V_k}}{\EE{N_k(t)}^2}\\
& \ge \frac{\delta_k V_k}{t+\beta}
- \frac{V_k \EE{D_{k,1} \ind{D_{k,1}>\beta}}}{t+\beta}
- 2B\frac{\delta_k \sqrt{V_k \delta_k}}{t^{3/2}}\,.
\end{align*}
\end{proof}

In the next result we give an upper bound on $\EE{(S(t)-(N(t)-1)\mu)^2}/\EE{N(t)-1}^2$.
Before stating this lemma, let us recall the definitions of $(I_m)_{m=1,2,\ldots}$, $(T_k(m))_{m=0,1,\ldots,1\le k \le K}$, $(Y_m)_{m=1,2,\ldots}$, $(J_m)_{m=0,1,2,\ldots}$ and $(N(t))_{t\ge 0}$:
 $I_m\in \{1,\ldots,K\}$ is the index of the sampler chosen by  $\A$ for round $m$;
 $T_k(m) = \sum_{s=1}^m \one{ I_s = k}$ is the number of samples obtained from sampler $k$
 by the end of round $m$ (the empty sum is defined as zero);
 $Y_m = X_{I_m,T_{I_m}(m)}$ is the $m$th sample observed by $\A$;
 $J_0=0$ and
 $J_{m+1}= \sum_{s=1}^m D_{I_s,T_{I_s}(s)}$ is the time when $\A$ observes the $(m+1)$th sample,
 and so the $m$th round lasts over the time period $[J_m,J_{m+1})$; and
 $N(t) = \min\cset{ m }{ J_{m}\ge t}$ is the index of the round at time $t$ (the indexing of rounds starts at one).
 Thus, $N(t)-1$ is the number of samples observed over the time period $[0,t]$.
 Note that $\sum_{k=1}^K T_k(m) = m$ for any $m\ge 0$ and thus, in particular,
 $\sum_{k=1}^K T_k(N(t)) = N(t)$.
 Further, remember that $S_m = \sum_{s=1}^m Y_m$ and $S(t) = S_{N(t)-1}$.
\begin{lemma}
\label{lem:SEN2}
Let~\Cref{ass:cost} hold and assume that
 the random variables $|X_{k,m}-\mu|_{k,m}$  are a.s. bounded by some constant $B>0$
 and that
 $k^*={\arg\min}_{1\le k \le K} \delta_k V_k$ is unique.
 For $s\ge 0$, let $f(s) = \max_{k\ne k^*} \EE{T_k(N(s)-1)}$ and assume that $c_f = \sup_{s>0} f(s)/s<+\infty$.
Suppose that $t \ge 2 \delta_{\max}$ and $\EE{N(t)-1}>1$. 
Then, 
\begin{align}\label{eq:vsbound}
\frac{\EE{(S(t)-(N(t)-1)\mu)^2}}{\EE{N(t)-1}^2} 
& \le \frac{\delta^* V^*}{t} +\frac{C}{t^{3/2}} + \frac{C' f(t)}{t^2}
\end{align}
for some constants $C,C'>0$ that depend only 
on the problem parameters $\delta_k,V_k$ and $c_f$. 
Furthermore,
\begin{equation}
\label{eq:Ntbound}
\EE{N(t)-1}  \ge \frac{t}{\delta_{\max}}-1\,.
\end{equation}
\end{lemma}
\begin{proof}
We proceed similarly to the proof of \cref{lem:SEN}.

We first prove \eqref{eq:Ntbound}.
By the definition of $N(t)$ and since $D_{k,m}$ is nonnegative for all $k$ and $m$,
\[
t \le J_{N(t)} \le \sum_{s=1}^{N(t)} D_{I_s,T_{I_s}(s)} = \sum_{k=1}^K \sum_{m=1}^{T_k(N(t))} D_{k,m}\,.
\]
Notice that $(N(t))_{t\ge 0}$, $(T_k(n))_{n=0,1,\ldots}$  are stopping times w.r.t. the filtration $(\FF_m; m \ge 1)$, where
$\FF_m = \sigma(I_1,J_1,\ldots,I_m,Y_m)$. Therefore, $T_k(N(t))$ is also a stopping time w.r.t. $(\FF_m)$.
Defining $\oT_k(t) = \EE{ T_k(N(t))}$,
Wald's  identity yields 
\begin{equation}
\label{eq:t_upper}
t \le \sum_{k=1}^K \EE{T_k(N(t))} \delta_k = \sum_{k=1}^K \oT_k(t) \delta_k\,.
\end{equation} 
Then,
\begin{equation*}
\EE{N(t)} = \sum_{k=1}^K \oT_k(t) 
 \ge \frac{1}{\delta_{\max}}\sum_{k=1}^K \oT_k(t) \delta_k 
 \ge \frac{t}{\delta_{\max}}\,,
\end{equation*}
finishing the proof of~\eqref{eq:Ntbound}.

Now, let us turn to showing that~\eqref{eq:vsbound} holds.
First, notice that 
\begin{align*}
S(t) - (N(t)-1)\mu& = - (Y_{N(t)}-\mu) +\left\{S_{N(t)} - N(t) \mu\right\}
 =- (Y_{N(t)}-\mu) + \sum_{k=1}^K S_{k,T_k(N(t))} - T_k(N(t))\mu.
\end{align*}
Hence, with the same calculation as in the proof of \cref{thm:regretmatch} (the difference is that here $N(t)$ is a stopping time
while $n$ is a constant in \cref{thm:regretmatch}),
we get 
\begin{equation*}
\EE{ \sum_{k=1}^K \left\{S_{k,T_k(N(t))} - T_k(N(t))\mu\right\}^2} = \sum_{k=1}^K \oT_k(t) V_k
\end{equation*}
and thus, introducing $F = S_{N(t)} - N(t) \mu$, $V_{\max} = \max_k V_k$, and $\bar{V}=\sum_{k\neq k^*} V_k$,
\begin{align}
\EE{ (S(t) - (N(t)-1)\mu)^2 } 
& =- 2 \EE{ (Y_{N(t)}-\mu) F } + \EE{ (Y_{N(t)}-\mu)^2 }  + \sum_{k=1}^K  \oT_k(t) V_k\nonumber\\
& \le 2 B (\EE{ F^2})^{1/2} + V_{\max} + \sum_{k=1}^K \oT_k(t) V_k\nonumber \\
& \le 2 B \sqrt{ \sum_{k=1}^K \oT_k(t) V_k } +  V_{\max} + \sum_{k=1}^K \oT_k(t) V_k\\
& \le 2B \sqrt{V_{\max} \sum_{k=1}^K \oT_k(t) } + V_{\max} + \sum_{k=1}^K \oT_k(t) V_k\\
& = 2B \sqrt{V_{\max} \EE{N(t)} } + V_{\max} + \sum_{k=1}^K \oT_k(t) V_k\,.
\label{eq:wald2a}
\end{align}
Introduce the notation $\bar{\delta}=\sum_{k\neq k^*} \delta_k$ and recall that, by assumption, $\oT_k(t) \le f(t)$ for all $k\neq k^*$.
On the other hand,  $\oT_{k^*}(t) \le \EE{N(t)}$ holds trivially.
Therefore, using \eqref{eq:wald2a}, \eqref{eq:t_upper}, and \eqref{eq:Ntbound}, 
we obtain
\begin{align*}
\frac{\EE{(S(t)-(N(t)-1)\mu)^2}}{\EE{N(t)-1}^2} 
& \le \frac{ \left(\sum_{k=1} \oT_k(t) \delta_k\right) \cdot 
\left\{\sum_{k=1}^K \oT_k(t) V_k  + 2B \sqrt{V_{\max} \EE{N(t)} } + V_{\max} \right\}}{t \,\EE{N(t)-1}^2} \\
& \le \frac{ \left(\oT_{k^*}(t) \delta^* + f(t) \bar{\delta} \right) 
\left\{\oT_{k^*} V_{k^*} + f(t) \bar{V}  + 2B \sqrt{V_{\max} \EE{N(t)-1}+1 } + V_{\max} \right\}}{t \,\EE{N(t)-1}^2} \\
& \le \frac{ \left(\EE{N(t)-1} \delta^*+\delta^* + f(t) \bar{\delta} \right) 
\left\{\EE{N(t)-1} V_{k^*} + f(t) \bar{V}  + 2B \sqrt{V_{\max} \EE{N(t)-1} } + C_0 \right\}}{t \,\EE{N(t)-1}^2} \\
& \le \frac{\delta^* V^*}{t} + \frac{C_1}{t \,\sqrt{\EE{N(t)-1}}} + \frac{C_2 f(t) + C_3}{t \,\EE{N(t)-1}}
+ \frac{C_4 f^2(t) + C_5 f(t) + C_6}{t \,\EE{N(t)-1}^2} \\
& \le \frac{\delta^* V^*}{t} + \frac{\delta_{\max}^{1/2} C_1}{t \sqrt{t-\delta_{\max}}} 
+\frac{\delta_{\max} (C_2 f(t) + C_3)}{t (t-\delta_{\max})} + \frac{\delta_{\max}^2(C_4 f^2(t) + C_5 f(t) + C_6)}{t(t-\delta_{\max})^2} \\
& \le \frac{\delta^* V^*}{t} +\frac{C}{t^{3/2}} + \frac{C'' f(t)}{t^2} + \frac{C''' f^2(t)}{t^3} \\
& \le \frac{\delta^* V^*}{t} +\frac{C}{t^{3/2}} + \frac{C' f(t)}{t^2}
\end{align*}
where $C_0,C_1,C_2,C_3,C_4,C_5,C_6,C,C',C'',C'''>0$ are appropriate constants depending on the problem parameters and  $c_f$, and in the first to last step we used $t \ge 2\delta_{\max}$ and $f(t) \le c_f t$ in the last step.
This finishes the proof of the second part of the lemma.
\end{proof}

Combining the above two lemmas the proof of \cref{thm:nonunifreduction} is straightforward.
We apply the first inequality of \cref{lem:SNconc} to $\hat{D} = \hmu(t)-\mu=S(t)/(N(t)-1)-\mu$
and the second inequality of the lemma to $\hat{D}^* = \hmu_{k^*}(t)-\mu=S_k(t)/(N_{k^*}(t)-1)-\mu$. 
It is easy to see that $\Prob{N(t)=1}$ and $\Prob{N_{k^*}(t)=1}$ both decay exponentially fast, so the difference of the first terms
is negligible. The difference of the second terms can be handled by \cref{lem:SEN,lem:SEN2} and the bounds \eqref{eq:nkbound}, \eqref{eq:Ntbound} on $\EE{N_{k^*}(t)}$,  $\EE{N(t)-1}$ 
(note that the lower bound on $\EE{N(t)-1}$ can be simplified to $t/(2\delta_{\max})$ using $t\ge 2\delta_{\max}$):
introducing $\lambda^*(\beta) = \EE{ D_{k^*,1} \ind{ D_{k^*,1}>\beta } }$ we obtain
\begin{align*}
\MoveEqLeft 
\frac{\EE{(S(t)- (N(t)-1)\mu)^2}}{\EE{N(t)-1}^2} \left(1+\frac{2c}{\sqrt{\EE{N(t)-1}}}\right)^2 
-	\frac{\EE{(S_{k^*}(t)-(N_{k^*}(t)-1)\mu)^2}}{\EE{N_{k^*}(t)-1}^2} \left(1-\frac{2c}{\sqrt{\EE{N_{k^*}(t)-1}}}\right)\\
&\le
 \left(\frac{\delta^* V^*}{t} + \frac{C}{t^{3/2}} + \frac{C' f(t)}{t^2}\right) \left(1+\frac{2c\sqrt{2\delta_{\max}}}{\sqrt{t}}\right)^2 
 -\left(
 	\frac{\delta^* V^*}{t+\beta} - 2B \frac{ \delta^* \sqrt{ \delta^* V^*}}{t^{3/2}} - \frac{V^* \lambda^*(\beta)}{t}
 	\right) \,
  \left(1-\frac{c\sqrt{2\delta^*}}{\sqrt{t}}\right)\\
&=
 \left(\frac{\delta^* V^*}{t} +  \frac{C}{t^{3/2}} + \frac{C' f(t)}{t^2}\right) \left(1+\frac{2c\sqrt{2\delta_{\max}}}{\sqrt{t}}
 + \frac{8c^2\delta_{\max}}{t} \right) 
 -\left(
 	\frac{\delta^* V^*}{t+\beta} - 2B \frac{ \delta^* \sqrt{ \delta^* V^*}}{t^{3/2}} - \frac{V^* \lambda^*(\beta)}{t}
 	\right) \,
  \left(1-\frac{c\sqrt{2\delta^*}}{\sqrt{t}}\right)\\
& \le \frac{\beta \,\delta^* V^*  }{t(t+\beta)} + 
		(c+C'') t^{-3/2}+C''' f(t) t^{-2} + C'''' \lambda^*(\beta) t^{-1}\\
& \le \delta^* V^* \, \beta t^{-2}  + 		
	(c+C'') t^{-3/2}+C''' f(t) t^{-2} + C'''' \lambda^*(\beta) t^{-1}\,,
\end{align*}
where we used that by assumption $c< \sqrt{t/(2\delta^*)}$. 
Using $\EE{N_{k^*}(t)-1} \ge t/(2\delta^*)$, the last term in the bound of \cref{lem:SNconc} for $\EE{(\hmu_{k^*}(t)-\mu)^2}$ becomes $C_1 t^{-2}\EE{ N_{k^*}(t)^2} \Prob{ N_{k^*}(t) > \EE{ N_{k^*}(t)} + c \sqrt{ \EE{N_{k^*}(t)-1}}} $. 
By $\EE{N(t)-1} \ge t/(2\delta_{\max})$, the indicator
in the third line of the bound for $\EE{(\hmu(t)-\mu)^2}$ can be bounded as $\one{\EE{N(t)-1}<4c^2} \le \one{ t < 8 \delta_{\max} c^2}$ which is zero since by assumption $c<\sqrt{t/(8\delta_{max})}$. 
Summing up our bounds and keeping the probability terms gives
\begin{align*}
\EE{ (\hmu(t)-\mu)^2 } - \EE{ (\hmu_{k^*}(t)-\mu)^2 }
& \le  \mu^2 \left\{\Prob{N(t)=1} - \Prob{N_{k^*}(t)=1}\right\}\\
& \qquad 
	+ \delta^* V^* \, \beta t^{-2}  + 		
	(c+C'') t^{-3/2}+C''' f(t) t^{-2} + C'''' \lambda^*(\beta) t^{-1}\\
& \qquad 	+	C_1 t^{-2}\EE{ N_{k^*}(t)^2} \Prob{ N_{k^*}(t) > \EE{ N_{k^*}(t)} + c \sqrt{ \EE{N_{k^*}(t)-1}}}  \\
& \qquad + C_2 \Prob{ N(t) < \EE{ N(t)} - c \sqrt{ \EE{N(t)-1}}}\,,
\end{align*}
thus, finishing the proof. 

\section{KL-Based Confidence Bound on Variance}
\label{sec:apx-klvar}

The following lemma gives a variance estimate based on Kullback-Leibler divergence. 

\begin{lemma}
\label{lem:KL-var-bound}
Let $Y_1,\ldots,Y_{2n}$ and $Q_1,\ldots,Q_n$ be two independent sequences of
independent and identically distributed random variables taking values in $[0,1]$. 
Furthermore,
for any $t=1,\ldots,n$, let
\begin{equation}
\label{eq:simplevarest}
\bvar_{2t} = \frac{1}{2t} \sum_{s=1}^t Q_t (Y_{2s}-Y_{2s-1})^2.
\end{equation}
Then, for any $\delta>0$,
\[
\P\left[\cup_{t=1}^n \left\{ \KL{2\bvar_{2t},2\EE{Q_1}\var[Y_1]} \ge \frac{\delta}{t}\right\} \right] \le 2 e \lceil \delta \log n \rceil e^{-\delta}~.
\]
\end{lemma}
\begin{proof}
Equation (5) of \citet{garivier13} states that if $Z_1,\ldots,Z_n$ are independent, identically distributed random variables taking values in  $[0,1]$ and
$\bar{Z}_t=\frac{1}{t}\sum_{s=1}^t Z_s$ for $t=1,\ldots,n$, then for any $\delta>0$,
\[
\P\left[\cup_{t=1}^n \left\{ \KL{\bar{Z}_t, \EE{Z_1}} \ge \frac{\delta}{t}\right\} \right] \le 2 e \lceil \delta \log n \rceil e^{-\delta}~.
\]
Defining $Z_t=Q_t(Y_{2t}-Y_{2t-1})^2$, we see that the above conditions on $Z_t$ are satisfied since they are clearly independent and identically distributed for $t=1,\ldots,n$, and
$Z_t \in [0,1]$ since $0 \le Q_t, Y_{2t-1}, Y_{2t} \le 1$. Now the
statement of the lemma follows since
$\EE{Z_t}=\EE{Q_t(Y_{2t}-Y_{2t-1})^2}=\EE{Q_t}\EE{((Y_{2t}-\EE{Y_{2t}}) - (Y_{2t-1} - \E{Y_{2t-1}}))^2}=2 \EE{Q_1}\var[Y_1]$.
\end{proof}
\begin{corollary}
\label{cor:kl-var-bound}
Let $Y_1,\ldots,Y_{2n}$ and $Q_1,\ldots,Q_n$ be two independent sequences of
independent and identically distributed random variables taking values in $[0,1]$. 
For any $t=2,\ldots,n$,
let $\bvar_t$ be defined by \eqref{eq:simplevarest} if $t$ is even, and let $\bvar_t=\bvar_{t-1}$ if $t$ is odd.
Furthermore, let
\[
\bvar_{t,min}=\inf\{ \mu: \KL{2\bvar_t, 2\mu} \le \delta/\lfloor t/2\rfloor \} 
\]
and
\[
\bvar_{t,max}=\sup\{ \mu: \KL{2\bvar_t, 2\mu} \le \delta/\lfloor t/2 \rfloor \}.
\]
Then for any $\delta>0$, with probability at least $1-2 e \lceil \delta \log \lfloor n/2 \rfloor \rceil e^{-\delta}$,
\[
\max_{2\le t \le n}\bvar_{t,min} \le \EE{Q}\var[Y] \le \min_{2\le t \le n} \bvar_{t,max}\,.
\]
\end{corollary}

\section{Details on Synthetic Experiments}
\label{sec:appendix-toy}

In considering alternative bounded payout distributions we evaluated
3 natural choices, truncated normal, uniform, and
\emph{scaled-Bernoulli}.  The latter non-standard distribution is a
transformation of a Bernoulli described by three parameters, a
midpoint $m \in \R$, a scale $s \in \R^+$, as well as a Bernoulli
parameter $p$.  Specifically, if $X \in \{0,1\}$ is a Bernoulli random variable 
with parameter $p$ then the transformation $m + (X -
0.5)s$ provides the corresponding scaled-Bernoulli sample.  We
conducted experiments using these different distributions, keeping
their variances the same, in an attempt to ascertain whether the shape
of the distribution might affect the performance of the various bandit
allocation algorithms.  However, we were unable to uncover an instance
where the shape of the distribution had a non-negligible effect on
performance. As a result, we decided to conduct our experiments using
scaled-Bernoulli distributions as they permit the maximum range for
the variance in an bounded interval.

\section{The Cox-Ingersoll-Ross Model}
\label{sec:appendix-finance-example}

In the Cox-Ingersoll-Ross
(CIR) model the interest rate at time $t$, denoted $r(t)$, follows a
\emph{square-root diffusion} given by a stochastic differential
equation of the form
\[
dr(t) = (\eta - \kappa r(t))dt + \sigma\sqrt{r(t)}dW(t)
,
\]
where $\eta$, $\kappa$ and $\sigma$ are fixed, 
problem-specific, constants and $W$
is a standard one-dimensional Brownian motion.  The payoff of an
option for this rate at time $T$ (maturity) is given as
\[
M\max(r(T) - K, 0)
,
\]
where the strike price $K$ and nominee amount $M$ are parameters of
the option.  The actual quantify of interest, the price of an option
at time $0$, is given as
\[
\bar P \defined \E\left[\exp(-\int_0^T r(t)dt)M\max(r(T) - K, 0)\right]
.
\]
The above integrals may approximated with the Monte Carlo method by
first discretizing the interest rate trajectory, $r(t)$, into $n$
points $r_{1:n}$ and simulating
\[
r_t = r_{t-1} + (\eta - \kappa r_{t-1})\frac{T}{n} + \sigma\sqrt{r_{t-1}\frac{T}{n}}\eps_t
,
\]
where $\eps_t \sim \mathcal{N}(0,1)$, and $r_0$ is given.  Given a set
of $N$ sampled trajectories, $r_{1:n}^{(1:N)}$, we may let
\[
p_i \defined 
\exp\left(
-\frac{T}{n}
\left(\frac{r_1^{(i)} + r_n^{(i)}}{2} + \sum_{t=1}^n
r_t^{(i)}\right)
\right) 
\;
M\max(r_n^{(i)} - K, 0)
,
\]
and approximate using $\bar P \approx \frac{1}{N}\sum_{i=1}^N p_i$. 

An important detail is that the value of the option is exactly zero
when the interest rate is below the strike price at maturity.
Consequently, when it comes to Monte Carlo simulation, we are less
interested in simulating trajectories with lower interest rates.  A
standard way of exploiting this intuition is an importance sampling
variant known as \emph{exponential twisting} (see
\citep{glasserman2003}).  Here, instead of sampling the noise $\eps_t$
directly from the \emph{target density}, $\mathcal{N}(0,1)$, one uses
a skewed \emph{proposal density}, $\mathcal{N}(\theta,1)$, defined by
a single \emph{drift} parameter, $\theta \in \R$. Deriving the
importance weights, $w_i \defined \exp(-\sum_{t=1}^n \epsilon_t^{(i)}
- \frac{n\theta^2}{2})$, then, we arrive with the unbiased
approximation $\bar P \approx \frac{1}{N}\sum_{i=1}^N w_ip_i$.

\section{Bayesian Logistic Regression Example}
\label{sec:appendix-lr-example}

We consider a standard Bayesian logistic regression model using a
multivariate Gaussian prior. 
To set up the model, 
for $y\in \{0,1\}$, $x\in \R^d$, $\theta\in \R^d$, 
let $p(y|x,\theta) = \sigma(\theta^\top x)^y (1-\sigma(\theta^\top x))^{1-y}$, 
where
$\sigma$ denotes the logit function
and let $p(\theta) = \mathcal{N}(\theta; 0,0.05I)$, i.e., the density of $d$ dimensional
Gaussian with mean $0$ and covariance $0.05 I$.
The labeled training examples $(x_1,y_1),\ldots,(x_T,y_T)$ 
are assumed to be generated as follows:
First, $\theta_* \sim p(\cdot)$ is chosen. The sequence 
$(x_1,y_1),\ldots,(x_T,y_T)$ 
is assumed to be i.i.d. given $\theta_*$ and the labels are assumed to satisfy
\[
\Prob{ y_t = 1|x_t,\theta_*} = p(y_t|x_t,\theta_*)\,,
\]
i.e., $y_t \sim \mathrm{Ber}( \sigma(\theta_*^\top x_t) )$.
The posterior distribution of $\theta_*$ given $x_{1:T}$ and $y_{1:T}$ then follows
\[
p_T(\theta) \defined 
\frac{1}{\mathcal{Z}} p_0(\theta) \prod_{t=1}^T p(y_t|x_t,\theta)\,
\]
where 
 the normalization constant is evaluated by integrating out
 $\theta$, that is
\[
\mathcal{Z} = \int p_0(\theta)  \prod_{t=1}^T p(y_t|x_t,\theta) \,\,d\theta
.
\]
In our demonstrations we consider approximating this integral using
\MC integration by drawing i.i.d. samples over $\theta_{i}$ for $i=1,...,N$ using
annealed importance sampling.  To do so we first define an annealed
target distribution 
$$
f_j(\theta) = f_0(\theta)^{\beta_j} f_n(\theta)^{1-\beta_j}
$$
for some $1 = \beta_{1} > \beta_2 > ... > \beta_n = 0$, where 
$f_n(\theta) = p_0(\theta)$, and posterior $f_0(\theta)
= p_T(\theta)$.  The $\beta_j$ values
are set using the \emph{power of 4} heuristic suggested by
\citep{kuss2005}, that is, given a fixed number of annealing steps $n$
we let $\beta_{1:n} = \{0, \frac{1}{n-1}, \frac{2}{n-1}, ..., 1\}^4$.

Additionally, AIS requires that we specify a sequence of MCMC
transitions to use at each annealing step.
For this we found that
slice sampling \citep{neal2003slice} moves were most effective
(compared to Metropolis-Hastings, Langevin, and Hamiltonian MCMC
\citep{Neal2010} moves) in additional to being effectively
parameter-free. If we let $T_i(x,x')$ denote the slice sampling
transition operator that meets detailed balance w.r.t. target $f_j$,
then the AIS algorithm is given by the following procedure:
\begin{itemize}
\item generate $\theta^{(n-1)} \sim f_n(\theta)$;
\item generate $\theta^{(n-2)}$ from $\theta^{(n-1)}$ using $T_{n-1}$;
\item ...
\item generate $\theta^{(1)}$ from $\theta^{(2)}$ using $T_{2}$;
\item generate $\theta^{(0)}$ from $\theta^{(1)}$ using $T_{1}$.
\end{itemize}
Defining 
$$
w = \frac{f_{n-1}(\theta^{(n-1)})}{f_{n}(\theta^{(n-1)})}\frac{f_{n-2}(\theta^{(n-2)})}{f_{n-1}(\theta^{(n-2)})}\cdots \frac{f_{1}(\theta^{(1)})}{f_{2}(\theta^{(1)})}\frac{f_{0}(\theta^{(0})}{f_{1}(\theta^{(0)})}
$$
we get that $\mathcal{Z} = \EE{ w}$.
\if
However, since we are using multiple estimators to allocate samples to
the different AIS parameterizations we would of course use the
combined estimator given in \cref{sec:uniformcost}. 
\fi

\end{document}